%% file: main.tex
\documentclass[runningheads]{llncs}
\usepackage{graphicx}
\usepackage{amsmath,amssymb} % define this before the line numbering.
\usepackage{color}
\usepackage{enumitem,kantlipsum}
\usepackage{hyperref}
\usepackage{epsfig}
\usepackage{graphicx}
\usepackage{amssymb}
\usepackage{url}            % simple URL typesetting
\usepackage{booktabs}       % professional-quality tables
\usepackage{amsfonts}       % blackboard math symbols
\usepackage{nicefrac}       % compact symbols for 1/2, etc.
\usepackage{microtype}      % microtypography
\usepackage{xcolor}
\usepackage{epsfig}
\usepackage{graphicx}
\usepackage{multirow}
\usepackage{bbm}
\usepackage{mathtools}
\usepackage{bm}
\usepackage{textcomp}
\usepackage{footnote}
\usepackage{tablefootnote}
\usepackage{algorithm}
\usepackage{algpseudocode}
\begin{document}
\title{Pairwise Confusion \\ for Fine-Grained Visual Classification}
% Replace with your title

\titlerunning{Pairwise Confusion for Fine-Grained Visual Classification}
% Replace with a meaningful short version of your title
%
\author{Abhimanyu Dubey\inst{1} \and
Otkrist Gupta\inst{1} \and
Pei Guo\inst{2} \and
Ramesh Raskar\inst{1} \and
Ryan Farrell\inst{2} \and
Nikhil Naik\inst{1,3}
}
%
%Please write out author names in full in the paper, i.e. full given and family names.
%If any authors have names that can be parsed into FirstName LastName in multiple ways, please include the correct parsing, in a comment to the volume editors:
%\index{Lastnames, Firstnames}
%(Do not uncomment it, because you may introduce extra index items if you do that, we will use scripts for introducing index entries...)
\authorrunning{A. Dubey, O. Gupta, P. Guo, R. Raskar, R. Farrell and N. Naik}
% Replace with shorter version of the author list. If there are more authors than fits a line, please use A. Author et al.
%

\institute{Massachusetts Institute of Technology, Cambridge MA 02139, USA \\
\email{\{dubeya,otkrist,raskar,naik\}@mit.edu}\and
Brigham Young University, Provo UT 84602, USA \\
\email{peiguo, farrell@cs.byu.edu} \and
Harvard University, Cambridge MA 02139, USA\\
\email{naik@fas.harvard.edu}}
\maketitle              % typeset the header of the contribution
\begin{abstract}
Fine-Grained Visual Classification (FGVC) datasets contain small sample sizes, along with significant intra-class variation and inter-class similarity. While prior work has addressed intra-class variation using localization and segmentation techniques, inter-class similarity may also affect feature learning and reduce classification performance. In this work, we address this problem using a novel optimization procedure for the end-to-end neural network training on FGVC tasks. Our procedure, called Pairwise Confusion (PC) reduces overfitting by intentionally {introducing confusion} in the activations. With PC regularization, we obtain state-of-the-art performance on six of the most widely-used FGVC datasets and demonstrate improved localization ability. {PC} is easy to implement, does not need excessive hyperparameter tuning during training, and does not add significant overhead during test time.
\end{abstract}

%%%%%%%%% BODY TEXT
\input{S_intro}
\input{S_related}
\input{S_method}
\input{S_experiments}
\input{S_results}

\newpage

\title{Pairwise Confusion for Fine-Grained Visual Classification : Supplementary Material}
% Replace with your title

\titlerunning{Pairwise Confusion : Supplementary Material}
% Replace with a meaningful short version of your title
%
\author{Abhimanyu Dubey\inst{1} \and
Otkrist Gupta\inst{1} \and
Pei Guo\inst{2} \and
Ramesh Raskar\inst{1} \and
Ryan Farrell\inst{2} \and
Nikhil Naik\inst{1,3}
}
%
%Please write out author names in full in the paper, i.e. full given and family names.
%If any authors have names that can be parsed into FirstName LastName in multiple ways, please include the correct parsing, in a comment to the volume editors:
%\index{Lastnames, Firstnames}
%(Do not uncomment it, because you may introduce extra index items if you do that, we will use scripts for introducing index entries...)
\authorrunning{A. Dubey, O. Gupta, P. Guo, R. Raskar, R. Farrell and N. Naik}
% Replace with shorter version of the author list. If there are more authors than fits a line, please use A. Author et al.
%

\institute{Massachusetts Institute of Technology, Cambridge MA 02139, USA \\
\email{\{dubeya,otkrist,raskar,naik\}@mit.edu}\and
Brigham Young University, Provo UT 84602, USA \\
\email{peiguo, farrell@cs.byu.edu} \and
Harvard University, Cambridge MA 02139, USA\\
\email{naik@fas.harvard.edu}}
\maketitle

\renewcommand{\thetable}{S\arabic{table}}
\renewcommand{\thefigure}{S\arabic{figure}}
\renewcommand{\thesection}{S\arabic{section}}

\section{Proofs for Lemmas from Section 3 in Main Text}
\subsection{Equation 4 from Main Text (Behavior of Jeffrey's Divergence)}
Consider Jeffrey's divergence with $N=2$ classes, and that $\mathbf x_1$ belongs to class 1, and $\mathbf x_2$ belongs to class 2. For a model with parameters $\theta$ that  correctly identifies both $\mathbf x_1$ and $\mathbf x_2$ by training using cross-entropy loss, $p_\theta(\mathbf y_1 | x_1) = 1 - \delta_1$ and $p_\theta(\mathbf y_2 | x_2) = 1 - \delta_2$, where $0 < \delta_1, \delta_2 < \frac{1}{2}$ (since the classifier outputs correct predictions for the input images), we get:
\begin{align}
\mathbb D_{\mathsf J} (p_\theta(\mathbf y | \mathbf x_1), p_\theta(\mathbf y | \mathbf x_2)) &= \begin{multlined}[t](1 - \delta_1 -\delta_2)\cdot(\log(\frac{(1-\delta_1)}{\delta_2})) \\ + (\delta_1 - 1 + \delta_2)\cdot(\log(\frac{\delta_1}{(1 - \delta_2)})) \end{multlined}\\
&= \begin{multlined}[t](1 - \delta_1 -\delta_2)\cdot(\log(\frac{(1-\delta_1)}{\delta_2})) \\ + (1 - \delta_1 - \delta_2)\cdot(\log(\frac{(1-\delta_2)}{\delta_1})) \end{multlined}\\
&= (1 - \delta_1 -\delta_2)\cdot(\log\frac{(1-\delta_1)(1-\delta_2)}{\delta_1\delta_2}) \\
&\geq (1 - \delta_1 -\delta_2)\cdot(2\log(1-\delta_1-\delta_2) -\log(\delta_1\delta_2)) \label{eqn:jef3}
\end{align}

\subsection{Lemmas 1  and 2 from Main Text (Euclidean Confusion Bounds)}
\begin{lemma}
\label{lemma_1}
On a finite probability space, for probability measures $P, Q$:
\begin{equation*}
\mathbb D_{\mathsf{EC}}(P, Q) \leq \mathbb D_{\mathsf J}(P, Q)
\end{equation*}
where $\mathbb D_{\mathsf J}(P, Q)$ is the Jeffrey's Divergence between $P$ and $Q$.
%$\mathbb D_{TV}$ denotes the total variation distance between the probability measures $P$ and $Q$, and
\end{lemma}
\begin{proof}
By the definition of Euclidean Confusion, we have:
\begin{align}
\mathbb D_{\mathsf{EC}}(P, Q) &=  \sum_{u \in \mathcal U} (p(u) - q(u))^2 \\ \intertext{For a finite-dimensional vector $x$, $\lVert x \rVert_2 \leq \lVert x \rVert_1$, therefore:}
&\leq (\sum_{u \in \mathcal U} | p(u) - q(u) | )^2 \\ \intertext{Since $\mathbb D_{\mathsf{TV}}(P, Q) = \frac{1}{2}(\sum_{u \in \mathcal U} | p(u) - q(u) |)$ for finite alphabet $\mathcal U$, we have:}
&= 4\mathbb D_{\mathsf{TV}}(P, Q)^2 \\ \intertext{Since Total Variation Distance is symmetric, we have:}
&= 2(\mathbb D_{\mathsf{TV}}(P, Q)^2 + \mathbb D_{\mathsf{TV}}(Q, P)^2) \\ \intertext{By Pinsker's Inequality, $\mathbb D_{\mathsf{TV}}(P, Q) \leq \sqrt{\frac{1}{2}\mathbb D_{\mathsf{KL}}(P || Q)}$, and similarly $\mathbb D_{\mathsf{TV}}(Q, P) \leq \sqrt{\frac{1}{2}\mathbb D_{\mathsf{KL}}(Q || P)}$, therefore:}
&\leq 2\Big(\frac{1}{2}\mathbb D_{\mathsf{KL}}(P || Q) + \frac{1}{2}\mathbb D_{\mathsf{KL}}(Q || P)\Big) \\
&= \mathbb D_{\mathsf J}(P, Q)
\end{align}
\end{proof}
\begin{lemma}
On a finite probability space, for probability measures $P, Q$:
\begin{equation*}
\mathbb D_{\mathsf{EC}}(P, Q) \leq 4\mathbb D_{\mathsf{TV}}(P, Q)^2
\end{equation*}
where $\mathbb D_{\mathsf{TV}}$ denotes the total variation distance between $P$ and $Q$.
\end{lemma}
\begin{proof}
By the definition of Euclidean Confusion, we have:
\begin{align}
\mathbb D_{\mathsf{EC}}(P, Q) &=  \sum_{u \in \mathcal U} (p(u) - q(u))^2 \\ \intertext{For a finite-dimensional vector $x$, $\lVert x \rVert_2 \leq \lVert x \rVert_1$, therefore:}
&\leq (\sum_{u \in \mathcal U} | p(u) - q(u) | )^2 \\ \intertext{Since $\mathbb D_{\mathsf{TV}}(P, Q) = \frac{1}{2}(\sum_{u \in \mathcal U} | p(u) - q(u) |)$ for finite alphabet $\mathcal U$, we have:}
&= 4\mathbb D_{\mathsf{TV}}(P, Q)^2
\end{align}
\end{proof}

\subsection{Proofs for Lemma 3 and Corollary 1 from the Main Text (Euclidean Confusion over Sets)}
\begin{definition}
In a standard classification setting with $N$ classes, consider a training set with $m = \sum_{i=1}^N m_i$ training examples, where $m_i$ denotes the number of training samples for class $i$. For simplicity of notation, let us denote the set of conditional probability distributions of all training points belonging to class $i$ for a model parameterized by $\theta$ as $\mathcal S_i = \{p_\theta(\mathbf y|\mathbf x^i_1), p_\theta(\mathbf y|\mathbf x^i_2), ..., p_\theta(\mathbf y|\mathbf x^i_{m_i})\}$. Then, for a model parameterized by $\theta$, the Euclidean Confusion is given by:
\begin{align}
\mathbb D_{\mathsf{EC}}(\mathcal S_i, \mathcal S_j; \theta) \triangleq \frac{1}{m_im_j} \Big(\sum_{u,v}^{m_i,m_j} \mathbb D_{\mathsf{EC}}(p_\theta(\mathbf y |\mathbf x^i_u), p_\theta(\mathbf y|\mathbf x^j_v)) \Big) \label{eqn:set_ec}
% = \frac{1}{m_im_j}\Big(\sum_{u,v}^{m_i,m_j} \lVert p_\theta(\mathbf y |\mathbf x^i_u) - p_\theta(\mathbf y|\mathbf x^j_v) \rVert_2^2 \Big)
\end{align}
\end{definition}
\begin{lemma}
For sets $\mathcal S_i$, $\mathcal S_j$ and $\mathbb D_{\mathsf{EC}}(\mathcal S_i, \mathcal S_j; \theta)$ as defined in Equation~(\ref{eqn:set_ec}):
\begin{equation*}
  \tfrac{1}{2}\mathbb D_{\mathsf{EN}}(\mathcal S_i, \mathcal S_j; \theta)^2 \leq \mathbb D_{\mathsf{EC}}(\mathcal S_i, \mathcal S_j; \theta)
\end{equation*}
where $\mathbb D_{\mathsf{EN}}(\mathcal S_i, \mathcal S_j; \theta)$ is the Energy Distance under Euclidean norm between $\mathcal S_i$ and $\mathcal S_j$ (parameterized by $\theta$), and random vectors are selected with uniform probability in both $\mathcal S_i$ and $\mathcal S_j$.
\end{lemma}
\begin{proof}
From the definition of Euclidean Confusion, we have:
\begin{align}
\mathbb D_{\mathsf{EC}}(\mathcal S_i, \mathcal S_j; \theta) &= \frac{1}{m_im_j} \Big(\sum_{u,v}^{m_i,m_j} \mathbb D_{\mathsf{EC}}(p_\theta(\mathbf y |\mathbf x^i_u), p_\theta(\mathbf y|\mathbf x^j_v)) \Big) \\
&= \frac{1}{m_im_j} \Big(\sum_{u,v}^{m_i,m_j} \lVert p_\theta(\mathbf y |\mathbf x^i_u) - p_\theta(\mathbf y|\mathbf x^j_v) \rVert_2^2 \Big) \\ \intertext{Considering $X_i \sim \text{Uniform}(\mathcal S_i)$, then we get:}
&= \frac{1}{m_j} \Big(\sum_{v}^{m_j} \mathbb E[\lVert X_i - p_\theta(\mathbf y|\mathbf x^j_v) \rVert_2^2] \Big) \\ \intertext{Considering $X_j \sim \text{Uniform}(\mathcal S_j)$, we obtain:}
&= \mathbb E [\lVert X_i - X_j \rVert_2^2] \label{eq:dec} \\ \intertext{Under the squared Euclidean norm distance, the Energy Distance can be given by:}
\mathbb D_{\mathsf{EN}}(\mathcal S_i, \mathcal S_j; \theta)^2 &= 2\mathbb E [\lVert X - Y \rVert_2^2] - \mathbb E [\lVert X - X' \rVert_2^2] - \mathbb E [\lVert Y - Y' \rVert_2^2] \\ \intertext{Where random variables $X, X' \sim \mathcal P(\mathcal S_i)$ and $Y, Y' \sim \mathcal P(\mathcal S_j)$. If $\mathcal P(\mathcal S_i) = \text{Uniform}(\mathcal S_i)$, and $\mathcal P(\mathcal S_j) = \text{Uniform}(\mathcal S_j)$, we have by substitution of Equation (\ref{eq:dec}):}
\frac{1}{2}\mathbb D_{\mathsf{EN}}(\mathcal S_i, \mathcal S_j; \theta)^2 &=\mathbb D_{\mathsf{EC}}(\mathcal S_i, \mathcal S_j; \theta) - \frac{1}{2}\Big(\mathbb E [\lVert X - X' \rVert_2^2] + \mathbb E [\lVert Y - Y' \rVert_2^2]\Big) \label{eq:ec_eq} \intertext{Since $\lVert x - y \rVert_2^2 \geq 0 \ \forall x \in \mathcal X, y \in \mathcal Y $; $\mathbb E_{x \sim \mathcal X, y \sim \mathcal Y}[\lVert x - y \rVert_2^2] \geq 0 \ \forall \ $ finite sets $\mathcal X, \mathcal Y $. Therefore, we have:}
\frac{1}{2}\mathbb D_{\mathsf{EN}}(\mathcal S_i, \mathcal S_j; \theta)^2 &\leq \mathbb D_{\mathsf{EC}}(\mathcal S_i, \mathcal S_j; \theta)
\end{align}
\end{proof}
\begin{corollary}
  For sets $\mathcal S_i$, $\mathcal S_j$ and $\mathbb D_{\mathsf{EC}}(\mathcal S_i, \mathcal S_j; \theta)$ as defined in Equation~(\ref{eqn:set_ec}), we have:
  \begin{equation*}
    \mathbb D_{\mathsf{EC}}(\mathcal S_i, \mathcal S_i; \theta) + \mathbb D_{\mathsf{EC}}(\mathcal S_j, \mathcal S_j; \theta) \leq 2\mathbb D_{\mathsf{EC}}(\mathcal S_i, \mathcal S_j; \theta)
  \end{equation*}
  with equality only when $\mathcal S_i = \mathcal S_j$.
\end{corollary}
\begin{proof}
From Equation~(\ref{eq:ec_eq}), we have:
\begin{gather}
\frac{1}{2}\mathbb D_{\mathsf{EN}}(\mathcal S_i, \mathcal S_j; \theta)^2 =\mathbb D_{\mathsf{EC}}(\mathcal S_i, \mathcal S_j; \theta) - \frac{1}{2}\Big(\mathbb E [\lVert X - X' \rVert_2^2] + \mathbb E [\lVert Y - Y' \rVert_2^2]\Big) \\ \intertext{From Equation~(\ref{eq:dec}), we have:}
\mathbb D_{\mathsf{EC}}(\mathcal S_i, \mathcal S_j; \theta) = \mathbb E [\lVert X_i - X_j \rVert_2^2] \\ \intertext{For $\mathcal S_i = \mathcal S_j$, we have with $X_i, X_j \sim \text{Uniform}(\mathcal S_i)$:}
\mathbb D_{\mathsf{EC}}(\mathcal S_i, \mathcal S_i; \theta)= \mathbb E [\lVert X_i - X_j \rVert_2^2] \\ \intertext{Replacing this in Equation~(\ref{eq:ec_eq}), we have with $X, X' \sim \text{Uniform}(\mathcal S_i)$ and $Y, Y' \sim \text{Uniform}(\mathcal S_j)$:}
\frac{1}{2}\mathbb D_{\mathsf{EN}}(\mathcal S_i, \mathcal S_j; \theta)^2 =\mathbb D_{\mathsf{EC}}(\mathcal S_i, \mathcal S_j; \theta) - \frac{1}{2}\Big(\mathbb E [\lVert X - X' \rVert_2^2] + \mathbb E [\lVert Y - Y' \rVert_2^2]\Big) \\
=\mathbb D_{\mathsf{EC}}(\mathcal S_i, \mathcal S_j; \theta) - \frac{1}{2}\Big(\mathbb D_{\mathsf{EC}}(\mathcal S_i, \mathcal S_i; \theta) + \mathbb D_{\mathsf{EC}}(\mathcal S_j, \mathcal S_j; \theta)\Big) \\ \intertext{From Szekely \textit{et al.}~\cite{szekely2013energy}, we know that the Energy Distance $\geq 0$ with equality if and only if $\mathcal S_i = \mathcal S_j$. Thus, we have that:}
\mathbb D_{\mathsf{EC}}(\mathcal S_i, \mathcal S_i; \theta) + \mathbb D_{\mathsf{EC}}(\mathcal S_j, \mathcal S_j; \theta)  \leq 2\mathbb D_{\mathsf{EC}}(\mathcal S_i, \mathcal S_j; \theta)
\end{gather}
With equality only when  $\mathcal S_i = \mathcal S_j$.
\end{proof}

\section{Training Details}
In this section, we describe the process for training with Pairwise Confusion for different base architectures, including the list of hyperparameters using for different datasets.

\textbf{ResNet-50:} In all experiments, we train for 40000 iterations with batch-size 8, with a linear decay of the learning rate from an initial value of 0.1. The hyperparameter for the confusion term for each dataset is given in Table \ref{tab:hyp_resnet}.
\begin{table}[h]
\centering
\begin{tabular}{|l|l|}
\hline
\textbf{Dataset} &  $\lambda$  \\ \hline
CUB2011          & 10                                       \\ \hline
NABirds          & 15                                   \\ \hline
Stanford Dogs    & 10                                     \\ \hline
Cars             & 10                                            \\ \hline
Flowers-102      & 10                                      \\ \hline
Aircraft         & 15                            \\ \hline
\end{tabular}
\caption{Regularization parameter $\lambda$ for ResNet-50 experiments.}
\label{tab:hyp_resnet}
\end{table}

\textbf{Bilinear and Compact Bilinear CNN:} In all experiments, we use the training procedure described by the authors\footnote{\url{https://github.com/gy20073/compact_bilinear_pooling/tree/master/caffe-20160312/examples/compact_bilinear}}. In addition, we repeat the described step 2 without the loss on confusion from the obtained weights after performing Step 2 with the loss, and obtain an additional 0.5 percent gain in performance. The hyperparameter for the confusion term for each dataset is given in Table  \ref{tab:hyp_bilinear}.

\begin{table}[h]
\centering
\begin{tabular}{|l|l|}
\hline
\textbf{Dataset} & $\lambda$  \\ \hline
CUB2011          & 20                            \\ \hline
NABirds          & 20                           \\ \hline
Stanford Dogs    & 10              \\ \hline
Cars             & 10                             \\ \hline
Flowers-102      & 10                           \\ \hline
Aircraft         & 10                           \\ \hline
\end{tabular}
\caption{Regularization parameter $\lambda$ for Bilinear CNN experiments.}
\label{tab:hyp_bilinear}
\end{table}

\textbf{DenseNet-161:} In all experiments, we train for 40000 iterations with batch-size 32, with a linear decay of the learning rate from an initial value of 0.1. The hyperparameter for the confusion term for each dataset is given in Table \ref{tab:hyp_densenet}.
\begin{table}[]
\centering
\begin{tabular}{|l|l|}
\hline
\textbf{Dataset} &  $\lambda$  \\ \hline
CUB2011          & 10                                       \\ \hline
NABirds          & 15                                   \\ \hline
Stanford Dogs    & 10                                     \\ \hline
Cars             & 15                                            \\ \hline
Flowers-102      & 10                                      \\ \hline
Aircraft         & 15                            \\ \hline
\end{tabular}
\caption{Regularization parameter $\lambda$ for DenseNet-161 experiments.}
\label{tab:hyp_densenet}
\end{table}

\textbf{GoogLeNet:} In all experiments, we train for 300000 iterations with batch-size 32, with a step size of 30000, decreasing it by a ratio of 0.96. The hyperparameter for the confusion term is given in Table \ref{tab:hyp_googlenet}.

\begin{table}[h!]
\centering
\begin{tabular}{|l|l|}
\hline
\textbf{Dataset} & $\lambda$ \\ \hline
CUB-200-2011          & 10                               \\ \hline
NABirds          & 20                                             \\ \hline
Stanford Dogs    & 10                                                 \\ \hline
Cars             & 10                                              \\ \hline
Flowers-102      & 10                                             \\ \hline
Aircraft         & 15                                        \\ \hline
\end{tabular}
\caption{Regularization parameter $\lambda$ for GoogLeNet experiments.}
\label{tab:hyp_googlenet}
\end{table}

\textbf{VGGNet-16:} In all experiments, we train for 40000 iterations with batch-size 32, with a linear decay of the learning rate from an initial value of 0.1. The hyperparameter for the confusion term is given in Table \ref{tab:hyp_vggnet}.
\begin{table}[h!]
\centering
\begin{tabular}{|l|l|}
\hline
\textbf{Dataset} &  $\lambda$  \\ \hline
CUB2011          & 15                                       \\ \hline
NABirds          & 15                                   \\ \hline
Stanford Dogs    & 10                                     \\ \hline
Cars             & 10                                            \\ \hline
Flowers-102      & 10                                      \\ \hline
Aircraft         & 15                            \\ \hline
\end{tabular}
\caption{Regularization parameter $\lambda$ for VGGNet-16 experiments.}
\label{tab:hyp_vggnet}
\end{table}

\section{Mean and Standard Deviation for FGVC Results}
In Table~\ref{tab:sup_fgvc}, we provide the mean and standard deviation values over five independent runs for training with Pairwise Confusion with different baseline models. These results correspond to Table 2 in the main text.

{\begin{table*}[t]
\centering
\scriptsize
\setlength\tabcolsep{0.75pt}
\begin{tabular}{lcc}
\multicolumn{2}{c}{(A) CUB-200-2011} \\ \hline
Method & Top-1 \\ \hline
GoogLeNet & 68.19  (0.39)\\
\textbf{PC}-GoogLeNet & 72.65  (0.47)\\ \hline
ResNet-50 & 78.15  (0.19)\\
\textbf{PC}-ResNet-50 & 80.21 (0.21)\\\hline
VGGNet16 &73.28 (0.41)\\
\textbf{PC}-VGGNet16 & 76.48 (0.43)\\ \hline
Bilinear CNN~\cite{lin2015bilinear} &84.10 (0.19)\\
\textbf{PC}-BilinearCNN & 85.58 (0.28)\\ \hline
DenseNet-161 & 84.21 (0.27) \\
\textbf{PC}-DenseNet-161 & 86.87 (0.35)\\ \hline\\
\end{tabular}
\hspace{2pt}
\begin{tabular}{lcc}
\multicolumn{2}{c}{(B) Cars} \\ \hline
Method & Top-1 \\ \hline
GoogLeNet & 85.65 (0.14)\\
\textbf{PC}-GoogLeNet & 86.91 (0.16) \\\hline
ResNet-50 & 91.71 (0.22)\\
\textbf{PC}-ResNet-50 & 93.43 (0.24) \\ \hline
VGGNet16 &80.60 (0.39) \\
\textbf{PC}-VGGNet16 & 83.16 (0.32)\\ \hline
Bilinear CNN~\cite{lin2015bilinear} &91.20 (0.18) \\
\textbf{PC}-Bilinear CNN & 92.45 (0.23) \\ \hline
DenseNet-161 & 91.83 (0.16) \\
\textbf{PC}-DenseNet-161 & 92.86 (0.18)  \\ \hline\\
\end{tabular}
\hspace{2pt}
\begin{tabular}{lcc}
\multicolumn{2}{c}{(C) Aircrafts} \\ \hline
Method & Top-1 \\ \hline
GoogLeNet & 74.04 (0.51) \\
\textbf{PC}-GoogLeNet & 78.86 (0.37) \\ \hline
ResNet-50 & 81.19 (0.28) \\
\textbf{PC}-ResNet-50 & 83.40 (0.25) \\ \hline
VGGNet16 &74.17 (0.21)\\
\textbf{PC}-VGGNet16 & 77.20 (0.24)\\ \hline
Bilinear CNN~\cite{lin2015bilinear} &84.10 (0.11)\\
\textbf{PC}-Bilinear CNN & 85.78 (0.13) \\ \hline
DenseNet-161 &86.30 (0.35) \\
\textbf{PC}-DenseNet-161 & 89.24 (0.32) \\ \hline\\
\end{tabular}
%%\vspace{30pt}
\begin{tabular}{lcc}
\multicolumn{2}{c}{(D) NABirds} \\ \hline
Method & Top-1 \\ \hline
GoogLeNet & 70.66 (0.17)\\
\textbf{PC}-GoogLeNet & 72.01 (0.14) \\ \hline
ResNet-50~& 63.55 (0.28) \\
\textbf{PC}-ResNet-50 & 68.15 (0.31) \\  \hline
VGGNet16 &68.34 (0.19)\\
\textbf{PC}-VGGNet16 & 72.25 (0.25)  \\ \hline
Bilinear CNN~\cite{lin2015bilinear} & 80.90 (0.09) \\
\textbf{PC}-Bilinear CNN & 82.01 (0.12) \\ \hline
DenseNet-161 & 79.35 (0.25)\\
\textbf{PC}-DenseNet-161 & 82.79 (0.20) \\ \hline
\end{tabular}
\hspace{1pt}
\begin{tabular}{lcc}
\multicolumn{2}{c}{(E) Flowers-102} \\ \hline
Method & Top-1 \\ \hline
GoogLeNet & 82.55 (0.11)\\
\textbf{PC}-GoogLeNet & 83.03 (0.15) \\\hline
ResNet-50 & 92.46 (0.14)\\
\textbf{PC}-ResNet-50 & 93.50 (0.12) \\ \hline
VGGNet16 &85.15 (0.08)\\
\textbf{PC}-VGGNet16 & 86.19 (0.07)\\ \hline
Bilinear CNN~\cite{lin2015bilinear} &92.52 (0.13)\\
\textbf{PC}-Bilinear CNN & 93.65 (0.18) \\\hline
DenseNet-161 & 90.07 (0.17)\\
\textbf{PC}-DenseNet-161 &91.39 (0.15) \\ \hline
\end{tabular}
\hspace{1pt}
\begin{tabular}{lcc}
\multicolumn{2}{c}{(E) Stanford Dogs} \\ \hline
Method & Top-1 \\ \hline
GoogLeNet & 55.76 (0.36) \\
\textbf{PC}-GoogLeNet & 60.61 (0.29) \\\hline
ResNet-50 & 69.92 (0.32)\\
\textbf{PC}-ResNet-50 & 73.35 (0.33) \\\hline
VGGNet16 &61.92 (0.40)\\
\textbf{PC}-VGGNet16 & 65.51 (0.42) \\ \hline
Bilinear CNN~\cite{lin2015bilinear} &82.13 (0.12)\\
\textbf{PC}-Bilinear CNN & 83.04 (0.09)\\  \hline
DenseNet-161 & 81.18 (0.27)\\
\textbf{PC}-DenseNet-161 & 83.75(0.28)\\ \hline
\end{tabular}
\vspace{5pt}
\caption{Pairwise Confusion (\textbf{PC}) obtains state-of-the-art performance on six widely-used fine-grained visual classification datasets (A-F). Improvement over the baseline model is reported as $(\Delta)$. All results averaged over 5 trials with standard deviations reported in parentheses.\label{tab:sup_fgvc}}
\end{table*}

\section{Comparison with Regularization}
We additionally compare the performance of our optimization technique with other regularization methods as well. We first compare Pairwise Comparison with with Label-Smoothing Regularization (LSR) on all six FGVC datasets for VGG-Net16, ResNet-50 and DenseNet-161.  These results are summarized in Table~\ref{tab:sup2}. Next, in Table~\ref{tab:sup1}, we compare the performance of Pairwise Confusion (\textbf{PC}) with several additional regularization techniques on the CIFAR-10 and CIFAR-100 datasets using two small architectures: CIFAR-10 Quick (C10Quick) and CIFAR-10 Full (C10Full), which are standard models available in the Caffe framework.

\begin{table*}[t]
  \centering
  \small
  \scalebox{0.9}{\begin{tabular}{c|c|cccccc}
    \hline \hline
    \multicolumn{2}{c}{\textbf{Method}} & CUB-200-2011 & Cars & Aircrafts & NABirds & Flowers-102 & Stanford Dogs \\ \hline
    \multirow{2}{*}{VGG-Net16} & \textbf{PC} & 72.65 & 83.16 & 77.20 & 72.25 & 86.19 & 65.51 \\ \cline{2-8}
     & LSR & 70.03 & 81.45 & 75.06 & 69.28 & 83.98 & 63.06 \\ \hline
     \multirow{2}{*}{ResNet-50} & \textbf{PC} & 80.21 & 93.43 & 83.40 & 68.15 & 93.50 & 73.35 \\ \cline{2-8}
      & LSR & 78.20 & 92.04 & 81.26 & 64.02 & 92.48 & 70.03 \\ \hline
      \multirow{2}{*}{DenseNet-161} & \textbf{PC} & 86.87 & 92.86 & 89.24 & 82.79 & 91.39 & 83.75 \\ \cline{2-8}
       & LSR & 84.86 & 91.96 & 87.05 & 80.11 & 90.24 & 85.68 \\ \hline
  \end{tabular}}
  \caption{Comparison with Label Smoothing Regularization (LSR)~\cite{szegedy2015going}.}
  \label{tab:sup2}
\end{table*}

\begin{table*}[!h]
\centering
\scalebox{0.63}{\begin{tabular}{c|ccc|ccc|ccc}
\hline
\hline
& \multicolumn{3}{c|}{CIFAR-10 on C10Quick} & \multicolumn{3}{c|}{CIFAR-10 on C10Full} & \multicolumn{3}{c}{CIFAR-100 on C10Quick} \\
Method & Train &Test &$\Delta$ & Train & Test & $\Delta$ & Train & Test & $\Delta$ \\
\hline
None &100.00 (0.00)  &75.54 (0.17) &  24.46 (0.23)& 95.15 (0.65)& 81.45 (0.22)& 14.65 (0.17) & 100.00 (0.03) & 42.41 (0.16) & 57.59 (0.29)\\
Weight-decay~\cite{krogh1991simple} &100.00 (0.00) &75.61 (0.18)&24.51 (0.34)&95.18 (0.19)& 81.53 (0.21)& 14.73  (0.20)& 100.00  (0.05)& 42.87  (0.19)& 57.13 (0.27)\\
DeCov \cite{cogswell2015reducing} \tablefootnote{Due to the lack of publicly available software implementations of DeCov, we are unable to report the performance of DeCov on CIFAR-10 Full.} & 88.78  (0.23)& 79.75  (0.17)& 8.04  (0.16)& - & - & - & 72.53  & 45.10 & 27.43 \\
Dropout~\cite{srivastava2014dropout} &99.5  (0.12)&79.41 (0.12) &20.09  (0.34) & 92.15 (0.19)& 82.40 (0.14)   & 9.81 (0.25)& 75.00 (0.11)& 45.89 (0.14)& 29.11 (0.20)\\ \hline
\textbf{PC} & 92.25 (0.14) & 80.51 (0.20)  & 10.74 (0.28) & 93.88 (0.21)& 82.67 (0.12)   & 11.21 (0.34)& 72.12 (0.05)& 46.72 (0.12)& 25.50 (0.14)\\\hline
\textbf{PC} + Dropout  &93.04 (0.19)  &\textbf{81.13} (0.22)& 11.01 (0.32)&93.85 (0.23)&83.57 (0.20) & 10.28 (0.27)& 71.15 (0.12)& \textbf{49.22} (0.08)& 21.93 (0.22)\\\hline
\end{tabular}}
\vspace{5pt}
\caption{Image classification performance and train-val gap ($\Delta$)) for Pairwise Confusion (\textbf{PC}) and popular regularization methods. The standard deviation across trials is mentioned in parentheses.}
\label{tab:sup1}
\end{table*}
\section{Changes to Class-wise Prediction Accuracy}
 We find that while the average and lowest per-class accuracy increase when training with PC, there is a small decline in top-performing class accuracy (See Table~\ref{table1}). Moreover, the standard deviation in per-class accuracy is reduced as well. We also found that using PC slightly increased false positive errors while obtaining a larger reduction in false negative errors. For example, on CUB-200-2011 with ResNet-50, the average false positive error is increased by 0.06\%, but the average false negative error is reduced by 0.13\%. So while some additional mistakes are made in terms of false positives, we curb/reduce the problem of classifier overconfidence by a larger margin.

 \begin{table}[t]
 \centering
 \begin{tabular}{l|c|c|c|c}
   \hline \hline
   %& \multicolumn{4}{c}{Average class-wise accuracy (\%)} \\ \cline{2-5}
 Class Accuracy              & Best  & Worst & Mean & Std. Dev.  \\
 \hline
 Baseline           & 91.14                          & 68.34                            & 78.15         & 5.12                     \\
 PC & 90.67                          & 70.95                            & 80.21         & 4.22                     \\ \hline \hline
 \end{tabular}
 \caption{Class-wise Performance Comparison on CUB-200-211 for ResNet-50}
 \label{table1}
 \end{table}
 
\bibliographystyle{splncs}
\bibliography{egbib}
\end{document}

%%%%%%%%% TITLE
\title{Supplementary Material : Optimizing Energy Distance for Fine Grained Visual Classification}
\newif\ifnote
\notetrue
\newcommand{\NNnote}[1]{\ifnote \textcolor{blue}{[{\em {\bf **Nikhil:} #1}]} \fi}
\newcommand{\ADnote}[1]{\ifnote \textcolor{red}{[{\em {\bf **Abhi:} #1}]} \fi}
\newcommand{\OGnote}[1]{\ifnote \textcolor{cyan}{[{\em {\bf **Otkrist:} #1}]} \fi}
\newcommand{\hsp}{\hspace{-12pt}}

\maketitle

\renewcommand{\thetable}{S\arabic{table}}
\renewcommand{\thefigure}{S\arabic{figure}}
\renewcommand{\thesection}{S\arabic{section}}

% \NNnote{1. "Please see the supplement for details on the mean and standard deviation values of performance, for all experiments.": Add a table that has std dev for all experiments in table 2 and also add results for VGGnet and Googlenet. 2. combine all lambda tables if its not too much work.}
Here we describe additional experimental and analytical content following our primary results in the main paper. We begin by describing the formulation details, training procedure for the models we use in our experiments along with additional results, and weighing parameters obtained through grid search.

\section{Energy Distance}
\subsection{Bounds on Pairwise Energy Distance Loss}
We can show that $\bm{\mathcal L}_p$ is a conservative estimator for the symmetric KL-divergence between two conditional probability distributions.
\begin{lemma}
  $\bm{\mathcal L}_p(\bm x_1, \bm x_2)$ provides a lower bound for the symmetric KL-divergence between two conditional probability distributions $\bm x_1$ and $\bm x_2$.
\end{lemma}
\begin{proof}
The symmetric KL-divergence $\mathbb D_S ( \bm x_1, \bm x_2)$ can be given by:
\begin{eqnarray}
\mathbb D_S ( \bm x_1, \bm x_2) = \mathbb{D}_{\mathsf{KL}}\infdivx{ \bm x_1 } {\bm x_2 } + \mathbb{D}_{\mathsf{KL}}\infdivx{ \bm x_2 } {\bm x_1 }
\end{eqnarray}
By Pinsker's Inequality, we know:
\begin{eqnarray}
2 \delta (\bm x, \bm y)^2 \leq \mathbb{D}_{\mathsf{KL}}\infdivx{ \bm x } {\bm y } \label{pinsker2}
\end{eqnarray}
for two distributions $\bm x$ and $\bm y$, and $\delta (\bm x, \bm y)$ denoting the total variation distance. In a finite probability space, we know that:
\begin{eqnarray}
2\delta (\bm x, \bm y) = || \bm x  - \bm y ||_1 = || \bm y - \bm x ||_1 \label{fp_eqn}
\end{eqnarray}
Squaring (\ref{fp_eqn}), replacing in (\ref{pinsker2}) and adding resulting equations for both directions of divergence, we get:
\begin{align}
\frac{1}{2}||\bm x_2 - \bm x_1||_1^2 + \frac{1}{2}||\bm x_1 - \bm x_2||_1^2 &\leq \mathbb D_S ( \bm x_1, \bm x_2)\\
||\bm x_2 - \bm x_1||_1^2 &\leq \mathbb D_S ( \bm x_1, \bm x_2)
\end{align}
For finite-dimensional $\bm x$, we know that $||\bm x||_2 \leq ||\bm x||_1$. Hence, we can write:
\begin{align}
||\bm x_2 - \bm x_1||_2^2 \leq ||\bm x_2 - \bm x_1||_1^2 &\leq \mathbb D_S ( \bm x_1, \bm x_2) \\
\implies \bm{\mathcal L}_p &\leq \mathbb D_S.
\end{align}
\end{proof}

\section{Training Details}
\subsection{GoogLeNet:}
We follow the same training procedure for all datasets as described here: We train for 300000 iterations with batch-size 32, with a step size of 30000, decreasing it by a ratio of 0.96. The hyperparameter for the confusion term is given in Table \ref{tab:hyp_googlenet}.

\begin{table}[h]
\centering
\begin{tabular}{|l|l|}
\hline
\textbf{Dataset} & $\lambda$ \\ \hline
CUB-200-2011          & 10                               \\ \hline
NABirds          & 20                                             \\ \hline
Stanford Dogs    & 10                                                 \\ \hline
Cars             & 10                                              \\ \hline
Flowers-102      & 10                                             \\ \hline
Aircraft         & 15                                        \\ \hline
\end{tabular}
\caption{Regularization parameter $\lambda$ for GoogLeNet experiments.}
\label{tab:hyp_googlenet}
\end{table}

\subsection{VGGNet-16:}
We follow the same training procedure for all datasets as described here: We train for 40000 iterations with batch-size 32, with a linear decay of the learning rate from an initial value of 0.1. The hyperparameter for the confusion term is given in Table \ref{tab:hyp_vggnet}.
\begin{table}[h]
\centering
\begin{tabular}{|l|l|}
\hline
\textbf{Dataset} &  $\lambda$  \\ \hline
CUB2011          & 15                                       \\ \hline
NABirds          & 15                                   \\ \hline
Stanford Dogs    & 10                                     \\ \hline
Cars             & 10                                            \\ \hline
Flowers-102      & 10                                      \\ \hline
Aircraft         & 15                            \\ \hline
\end{tabular}
\caption{Regularization parameter $\lambda$ for VGGNet-16 experiments.}
\label{tab:hyp_vggnet}
\end{table}

\subsection{ResNet-50:}
We follow the same training procedure for all datasets as described here: We train for 40000 iterations with batch-size 8, with a linear decay of the learning rate from an initial value of 0.1. The hyperparameter for the confusion term is given in Table \ref{tab:hyp_resnet}.
\begin{table}[h]
\centering
\begin{tabular}{|l|l|}
\hline
\textbf{Dataset} &  $\lambda$  \\ \hline
CUB2011          & 10                                       \\ \hline
NABirds          & 15                                   \\ \hline
Stanford Dogs    & 10                                     \\ \hline
Cars             & 10                                            \\ \hline
Flowers-102      & 10                                      \\ \hline
Aircraft         & 15                            \\ \hline
\end{tabular}
\caption{Regularization parameter $\lambda$ for ResNet-50 experiments.}
\label{tab:hyp_resnet}
\end{table}

\subsection{Bilinear and Compact Bilinear CNN:}
We follow the 2-step training procedure for all datasets as described in the following URL: \url{https://github.com/gy20073/compact_bilinear_pooling/tree/master/caffe-20160312/examples/compact_bilinear}. In addition, we repeat the described step 2 without the loss for all experiments on confusion from the obtained weights after performing Step 2 with the loss, and obtain an additional 0.5 percent gain in performance. We describe the hyperparameter choice in Table \ref{tab:hyp_bilinear}.
\begin{table}[h]
\centering
\begin{tabular}{|l|l|}
\hline
\textbf{Dataset} & $\lambda$  \\ \hline
CUB2011          & 20                            \\ \hline
NABirds          & 20                           \\ \hline
Stanford Dogs    & 10              \\ \hline
Cars             & 10                             \\ \hline
Flowers-102      & 10                           \\ \hline
Aircraft      	 & 10                           \\ \hline
\end{tabular}
\caption{Regularization parameter $\lambda$ for Bilinear CNN experiments.}
\label{tab:hyp_bilinear}
\end{table}

\subsection{DenseNet-161:}
We follow the same training procedure for all datasets as described here: We train for 40000 iterations with batch-size 32, with a linear decay of the learning rate from an initial value of 0.1. The hyperparameter for the confusion term is given in Table \ref{tab:hyp_densenet}.
\begin{table}[]
\centering
\begin{tabular}{|l|l|}
\hline
\textbf{Dataset} &  $\lambda$  \\ \hline
CUB2011          & 10                                       \\ \hline
NABirds          & 15                                   \\ \hline
Stanford Dogs    & 10                                     \\ \hline
Cars             & 15                                            \\ \hline
Flowers-102      & 10                                      \\ \hline
Aircraft         & 15                            \\ \hline
\end{tabular}
\caption{Regularization parameter $\lambda$ for DenseNet-161 experiments.}
\label{tab:hyp_densenet}
\end{table}

\section{Performance Details}
Here we provide complete results for different models (along with standard deviation) for each of the experiments.
\subsection{CUB-200-2011} Table~\ref{tab:sup3} summarizes the results for CUB-200-2011 with 5 baseline models.

\begin{table}[H]
\centering
\begin{tabular}{lcc}
\multicolumn{2}{c}{CUB-200-2011} \\ \hline
Method & Top-1 \\ \hline
GoogLeNet & 68.19  (0.39)\\
\textbf{PEDM}-GoogLeNet & 72.65  (0.47)\\ \hline
ResNet-50 & 78.15  (0.19)\\
\textbf{PEDM}-ResNet-50 & 80.21 (0.21)\\\hline
VGGNet16 &73.28 (0.41)\\
\textbf{PEDM}-VGGNet16 & 76.48 (0.43)\\ \hline
Bilinear CNN~\cite{lin2015bilinear} &84.10 (0.19)\\
\textbf{PEDM}-BilinearCNN & 85.58 (0.28)\\ \hline
DenseNet-161 & 84.21 (0.27) \\
\textbf{PEDM}-DenseNet-161 & 86.87 (0.35)\\ \hline\\
\end{tabular}
\caption{Performance shown (with deviation across trials in parenthesis) for CUB-200-2011 dataset.}
\label{tab:sup3}
\end{table}

\subsection{Cars} Table~\ref{tab:sup4} summarizes the results for Cars with 5 baseline models.
\begin{table}[H]
\centering
\begin{tabular}{lcc}
\multicolumn{2}{c}{Cars} \\ \hline
Method & Top-1 \\ \hline
GoogLeNet & 85.65 (0.14)\\
\textbf{PEDM}-GoogLeNet & 86.91 (0.16) \\\hline
ResNet-50 & 91.71 (0.22)\\
\textbf{PEDM}-ResNet-50 & 93.43 (0.24) \\ \hline
VGGNet16 &80.60 (0.39) \\
\textbf{PEDM}-VGGNet16 & 83.16 (0.32)\\ \hline
Bilinear CNN~\cite{lin2015bilinear} &91.20 (0.18) \\
\textbf{PEDM}-Bilinear CNN & 92.45 (0.23) \\ \hline
DenseNet-161 & 91.83 (0.16) \\
\textbf{PEDM}-DenseNet-161 & 92.86 (0.18)  \\ \hline\\
\end{tabular}
\caption{Performance shown (with deviation across trials in parenthesis) for Cars dataset.}
\label{tab:sup4}
\end{table}
\subsection{Aircrafts} Table~\ref{tab:sup5} summarizes the results for Aircrafts with 5 baseline models.
\begin{table}[H]
\centering
\begin{tabular}{lcc}
\multicolumn{2}{c}{Aircrafts} \\ \hline
Method & Top-1 \\ \hline
GoogLeNet & 74.04 (0.51) \\
\textbf{PEDM}-GoogLeNet & 78.86 (0.37) \\ \hline
ResNet-50 & 81.19 (0.28) \\
\textbf{PEDM}-ResNet-50 & 83.40 (0.25) \\ \hline
VGGNet16 &74.17 (0.21)\\
\textbf{PEDM}-VGGNet16 & 77.20 (0.24)\\ \hline
Bilinear CNN~\cite{lin2015bilinear} &84.10 (0.11)\\
\textbf{PEDM}-Bilinear CNN & 85.78 (0.13) \\ \hline
DenseNet-161 &86.30 (0.35) \\
\textbf{PEDM}-DenseNet-161 & 89.24 (0.32) \\ \hline\\
\end{tabular}
\caption{Performance shown (with deviation across trials in parenthesis) for Aircrafts dataset.}
\label{tab:sup5}
\end{table}
\subsection{NABirds} Table~\ref{tab:sup6} summarizes the results for NABirds with 5 baseline models.
\begin{table}[H]
\centering
\begin{tabular}{lcc}
\multicolumn{2}{c}{NABirds} \\ \hline
Method & Top-1 \\ \hline
GoogLeNet & 70.66 (0.17)\\
\textbf{PEDM}-GoogLeNet & 72.01 (0.14) \\ \hline
ResNet-50~& 63.55 (0.28) \\
\textbf{PEDM}-ResNet-50 & 68.15 (0.31) \\  \hline
VGGNet16 &68.34 (0.19)\\
\textbf{PEDM}-VGGNet16 & 72.25 (0.25)  \\ \hline
Bilinear CNN~\cite{lin2015bilinear} & 80.90 (0.09) \\
\textbf{PEDM}-Bilinear CNN & 82.01 (0.12) \\ \hline
DenseNet-161 & 79.35 (0.25)\\
\textbf{PEDM}-DenseNet-161 & 82.79 (0.20) \\ \hline
\end{tabular}
\caption{Performance shown (with deviation across trials in parenthesis) for NABirds dataset.}
\label{tab:sup6}
\end{table}

\subsection{Flowers-102} Table~\ref{tab:sup7} summarizes the results for NABirds with 5 baseline models.
\begin{table}[H]
\centering
\begin{tabular}{lcc}
\multicolumn{2}{c}{Flowers-102} \\ \hline
Method & Top-1 \\ \hline
GoogLeNet & 82.55 (0.11)\\
\textbf{PEDM}-GoogLeNet & 83.03 (0.15) \\\hline
ResNet-50 & 92.46 (0.14)\\
\textbf{PEDM}-ResNet-50 & 93.50 (0.12) \\ \hline
VGGNet16 &85.15 (0.08)\\
\textbf{PEDM}-VGGNet16 & 86.19 (0.07)\\ \hline
Bilinear CNN~\cite{lin2015bilinear} &92.52 (0.13)\\
\textbf{PEDM}-Bilinear CNN & 93.65 (0.18) \\\hline
DenseNet-161 & 90.07 (0.17)\\
\textbf{PEDM}-DenseNet-161 &91.39 (0.15) \\ \hline
\end{tabular}
\caption{Performance shown (with deviation across trials in parenthesis) for Flowers-102 dataset.}
\label{tab:sup7}
\end{table}
\subsection{Stanford Dogs} Table~\ref{tab:sup8} summarizes the results for Stanford Dogs with 5 baseline models.
\begin{table}[H]
\centering
\begin{tabular}{lcc}
\multicolumn{2}{c}{Stanford Dogs} \\ \hline
Method & Top-1 \\ \hline
GoogLeNet & 55.76 (0.36) \\
\textbf{PEDM}-GoogLeNet & 60.61 (0.29) \\\hline
ResNet-50 & 69.92 (0.32)\\
\textbf{PEDM}-ResNet-50 & 73.35 (0.33) \\\hline
VGGNet16 &61.92 (0.40)\\
\textbf{PEDM}-VGGNet16 & 65.51 (0.42) \\ \hline
Bilinear CNN~\cite{lin2015bilinear} &82.13 (0.12)\\
\textbf{PEDM}-Bilinear CNN & 83.04 (0.09)\\  \hline
DenseNet-161 & 81.18 (0.27)\\
\textbf{PEDM}-DenseNet-161 & 83.75(0.28)\\ \hline
\end{tabular}
\caption{Performance shown (with deviation across trials in parenthesis) for Stanford Dogs dataset.}
\label{tab:sup8}
\end{table}
\section{Comparison with Regularization}
We additionally compare the performance of our optimization technique with other regularization methods as well. In Table~\ref{tab:sup1} we first compare the performance of \textbf{PEDM} with several additional regularization techniques on the CIFAR-10 and CIFAR-100 datasets using two small architectures: CIFAR-10 Quick (C10Quick) and CIFAR-10 Full (C10Full), which are standard models available in the Caffe framework.

We also compare with Label-Smoothing Regularization (LSR) on all six FGVC datasets for VGG-Net16, ResNet-50 and DenseNet-161. These results are summarized in Table~\ref{tab:sup2}.
\begin{table*}[!h]
\centering
\scalebox{0.76}{\begin{tabular}{c|ccc|ccc|ccc}
\hline
\hline
& \multicolumn{3}{c|}{CIFAR-10 on C10Quick} & \multicolumn{3}{c|}{CIFAR-10 on C10Full} & \multicolumn{3}{c}{CIFAR-100 on C10Quick} \\
Method & Train &Test &$\Delta$ & Train & Test & $\Delta$ & Train & Test & $\Delta$ \\
\hline
None &100.00 (0.00)  &75.54 (0.17) &  24.46 (0.23)& 95.15 (0.65)& 81.45 (0.22)& 14.65 (0.17) & 100.00 (0.03) & 42.41 (0.16) & 57.59 (0.29)\\
Weight-decay~\cite{krogh1991simple} &100.00 (0.00) &75.61 (0.18)&24.51 (0.34)&95.18 (0.19)& 81.53 (0.21)& 14.73  (0.20)& 100.00  (0.05)& 42.87  (0.19)& 57.13 (0.27)\\
DeCov \cite{cogswell2015reducing} \tablefootnote{Due to the lack of publicly available software implementations of DeCov, we are unable to report the performance of DeCov on CIFAR-10 Full.} & 88.78  (0.23)& 79.75  (0.17)& 8.04  (0.16)& - & - & - & 72.53  & 45.10 & 27.43 \\
Dropout~\cite{srivastava2014dropout} &99.5  (0.12)&79.41 (0.12) &20.09  (0.34) & 92.15 (0.19)& 82.40 (0.14)   & 9.81 (0.25)& 75.00 (0.11)& 45.89 (0.14)& 29.11 (0.20)\\ \hline
\textbf{PEDM} & 92.25 (0.14) & 80.51 (0.20)  & 10.74 (0.28) & 93.88 (0.21)& 82.67 (0.12)   & 11.21 (0.34)& 72.12 (0.05)& 46.72 (0.12)& 25.50 (0.14)\\\hline
\textbf{PEDM} + Dropout  &93.04 (0.19)  &\textbf{81.13} (0.22)& 11.01 (0.32)&93.85 (0.23)&83.57 (0.20) & 10.28 (0.27)& 71.15 (0.12)& \textbf{49.22} (0.08)& 21.93 (0.22)\\\hline
\end{tabular}}
\caption{Image classification performance and train-val gap ($\Delta$)) for \textbf{PEDM} and popular regularization methods. The standard deviation across trials is mentioned in parentheses.}
\label{tab:sup1}
\end{table*}

\begin{table*}[t]
  \centering
  \small
  \begin{tabular}{c|c|cccccc}
    \hline \hline
    \multicolumn{2}{c}{\textbf{Method}} & CUB-200-2011 & Cars & Aircrafts & NABirds & Flowers-102 & Stanford Dogs \\ \hline
    \multirow{2}{*}{VGG-Net16} & \textbf{PEDM} & 72.65 & 83.16 & 77.20 & 72.25 & 86.19 & 65.51 \\ \cline{2-8}
     & LSR & 70.03 & 81.45 & 75.06 & 69.28 & 83.98 & 63.06 \\ \hline
     \multirow{2}{*}{ResNet-50} & \textbf{PEDM} & 80.21 & 93.43 & 83.40 & 68.15 & 93.50 & 73.35 \\ \cline{2-8}
      & LSR & 78.20 & 92.04 & 81.26 & 64.02 & 92.48 & 70.03 \\ \hline
      \multirow{2}{*}{DenseNet-161} & \textbf{PEDM} & 86.87 & 92.86 & 89.24 & 82.79 & 91.39 & 83.75 \\ \cline{2-8}
       & LSR & 84.86 & 91.96 & 87.05 & 80.11 & 90.24 & 85.68 \\ \hline
  \end{tabular}
  \caption{Comparison with Label Smoothing Regularization (LSR)~\cite{szegedy2015going}.}
  \label{tab:sup2}
\end{table*}

{\small
\bibliographystyle{ieee}
\bibliography{egbib}
}

%% file: S_intro.tex
\section{Introduction}
The Fine-Grained Visual Classification (FGVC) task focuses on differentiating between hard-to-distinguish object classes, such as species of birds, flowers, or animals; and identifying the makes or models of vehicles. FGVC datasets depart from conventional image classification in that they typically require expert knowledge, rather than crowdsourcing, for gathering annotations. FGVC datasets contain images with much higher visual similarity than those in large-scale visual classification (LSVC).  Moreover, FGVC datasets have minute inter-class visual differences in addition to the variations in pose, lighting and viewpoint found in LSVC~\cite{lin2015bilinear}. Additionally, FGVC datasets often exhibit long tails in the data distribution, since the difficulty of obtaining examples of different classes may vary. This combination of small, non-uniform datasets and subtle inter-class differences makes FGVC challenging even for powerful deep learning algorithms.

Most of the prior work in FGVC has focused on tackling the \textit{intra-class} variation in pose, lighting, and viewpoint using localization techniques~\cite{lin2015bilinear,jaderberg2015spatial,zhang2016weakly,krause2015fine,zhang2015fine}, and by augmenting training datasets with additional data from the Web~\cite{krause2016unreasonable,cui2016fine}. However, we observe that prior work in FGVC does not pay much attention to the problems that may arise due to the \textit{inter-class} visual similarity in the feature extraction pipeline. Similar to LSVC tasks, neural networks for  FGVC tasks are typically trained with cross-entropy loss~\cite{lin2015bilinear,cui2016fine,lin2017improved,cui2017kernel}. In LSVC datasets such as ImageNet~\cite{imagenet_cvpr09}, strongly discriminative learning using the cross-entropy loss is successful in part due to the significant inter-class variation (compared to intra-class variation), which enables deep networks to learn generalized discriminatory features with large amounts of data.

We posit that this formulation may not be ideal for FGVC, which shows smaller visual differences between classes and larger differences within each class than LSVC. For instance, if two samples in the training set have very similar visual content but different class labels, minimizing the cross-entropy loss will force the neural network to learn features that distinguish these two images with high confidence---potentially forcing the network to learn sample-specific artifacts for visually confusing classes in order to minimize training error. We suspect that this effect would be especially pronounced in FGVC, since there are fewer samples from which the network can learn generalizable class-specific features.

Based on this hypothesis, we propose that introducing \textit{confusion} in output logit activations during training for an FGVC task will force the network to learn slightly less discriminative features, thereby preventing it from overfitting to sample-specific artifacts. Specifically, we aim to \textit{confuse} the network, by minimizing the distance between the predicted probability distributions for random pairs of samples from the training set. To do so, we propose Pairwise Confusion (PC)\footnote{Implementation available at \href{https://github.com/abhimanyudubey/confusion}{\texttt{https://github.com/abhimanyudubey/confusion}}.}, a pairwise algorithm for training convolutional neural networks (CNNs) end-to-end for fine-grained visual classification.

In Pairwise Confusion, we construct a Siamese neural network trained with a novel loss function that attempts to bring class conditional probability distributions closer to each other. Using Pairwise Confusion with a standard network architecture like DenseNet~\cite{huang2016densely} or ResNet~\cite{he2016deep} as a base network, we obtain state-of-the-art performance on six of the most widely-used fine-grained recognition datasets, improving over the previous-best published methods by 1.86\% on average. In addition, PC-trained networks show better localization performance as compared to standard networks. Pairwise Confusion is simple to implement, has no added overhead in training or prediction time, and provides performance improvements both in FGVC tasks and other tasks that involve transfer learning with small amounts of training data.

%% file: S_related.tex
\section{Related Work}
% \begin{figure}[t]
% \includegraphics[width=\linewidth]{figures/data_samples.pdf}
% \caption{Samples from the Stanford Dogs~\cite{khosla2011novel} dataset. Figure (a) shows two very different samples from the same class (`Beagle'), and (b) shows two very similar samples from different classes (`Norwich Terrier' and `Norfolk Terrier').}
% \label{fig:data_samples}
% \end{figure}
\textbf{Fine-Grained Visual Classification:} Early FGVC research focused on methods to train with limited labeled data and traditional image features. Yao et al.~\cite{yao2011combining} combined strongly discriminative image patches with randomization techniques to prevent overfitting. Yao et al.~\cite{yao2012codebook} subsequently utilized template matching to avoid the need for a large number of annotations. 

Recently, improved localization of the target object in training images has been shown to be useful for FGVC~\cite{lin2015bilinear,zhang2014part,gao2016compact,Wang_2016_CVPR}. Zhang et al.~\cite{zhang2014part} utilize part-based Region-CNNs~\cite{ren2015faster} to perform finer localization. Spatial Transformer Networks~\cite{jaderberg2015spatial} show that learning a content-based affine transformation layer improves FGVC performance. Pose-normalized CNNs have also been shown to be effective at FGVC~\cite{branson2014bird,zhang2012pose}. Model ensembling and boosting has also improved performance on FGVC~\cite{Moghimi2016}. Lin et al.~\cite{lin2015bilinear} introduced Bilinear Pooling, which combines pairwise local feature sets and improves classification performance. Bilinear Pooling has been extended by Gao et al.~\cite{gao2016compact} using a compact bilinear representation and Cui et al.~\cite{cui2017kernel} using a general Kernel-based pooling framework that captures higher-order interactions of features.

\textbf{Pairwise Learning:} Chopra et al.~\cite{chopra2005learning} introduced a Siamese neural network for handwriting recognition.
Parikh and Grauman~\cite{parikh2011relative} developed a pairwise ranking scheme for relative attribute learning. Subsequently, pairwise neural network models have become common for attribute modeling~\cite{dubey2017modeling,souri2016deep,dubey2016deep,singh2016end}. 

\textbf{Learning from Label Confusion:} Our method aims to improve classification performance by introducing confusion within the output labels. Prior work in this area includes  methods that utilize label noise (e.g.,~\cite{reed2014training}) and data noise (e.g.,~\cite{xiao2015learning}) in training. Krause et al.~\cite{krause2016unreasonable} utilized noisy training data for FGVC. Neelakantan et al.~\cite{neelakantan2015adding} added noise to the gradient during training to improve generalization performance in very deep networks. Szegedy et al.~\cite{szegedy2016rethinking} introduced label-smoothing regularization for training deep Inception models.

In this paper, we bring together concepts from pairwise learning and label confusion and take a step towards solving the problems of overfitting and sample-specific artifacts when training neural networks for FGVC tasks.

%% file: S_method.tex
\section{Method}
\begin{table}[t]
\centering\footnotesize
\caption{A comparison of fine-grained visual classification (FGVC) datasets with large-scale visual classification (LSVC) datasets. FGVC datasets are significantly smaller and noisier than LSVC datasets.\label{tab:dataset_comparison}}
\setlength\tabcolsep{2.5pt}
\begin{tabular}{lcc} \hline \hline
\multirow{2}{*}{Dataset} & num. & samples\\
& classes & per class\\ \hline
Flowers-102~\cite{nilsback2008automated} & 102 & 10 \\
CUB-200-2011~\cite{wah2011caltech} & 200 & 29.97 \\
Cars~\cite{krause20133d} & 196 & 41.55 \\
NABirds~\cite{van2015building} & 550 & 43.5 \\
Aircrafts~\cite{maji2013fine} & 100 & 100 \\
Stanford Dogs~\cite{khosla2011novel} & 120 & 100 \\ \hline \hline
\end{tabular}
\hspace{10pt}
\begin{tabular}{lcc} \hline \hline
\multirow{2}{*}{Dataset} & num. & samples\\
& classes & per class\\ \hline
CIFAR-100~\cite{krizhevsky2014cifar} & 100 & 500 \\
ImageNet~\cite{imagenet_cvpr09} & 1000 & 1200 \\
CIFAR-10~\cite{krizhevsky2014cifar} & 10 & 5000 \\
SVHN~\cite{netzer2011reading} & 10 & 7325.7 \\ \hline \hline
\end{tabular}
\end{table}

FGVC datasets in computer vision are orders of magnitude smaller than LSVC datasets and contain greater imbalance across classes (see Table~\ref{tab:dataset_comparison}). Moreover, the samples of a class are not accurately representative of the complete variation in the visual class itself. The smaller dataset size can result in overfitting when training deep neural architectures with large number of parameters---even with preliminary layers being frozen. In addition, the training data may not be completely representative of the real-world data, with issues such as more abundant sampling for certain classes.  For example, in FGVC of birds, certain species from geographically accessible areas may be overrepresented in the training dataset. As a result, the neural network may learn to latch on to sample-specific artifacts in the image, instead of learning a versatile representation for the target object. We aim to solve both of these issues in FGVC (overfitting and sample-specific artifacts) by bringing the different class-conditional probability distributions closer together and \textit{confusing} the deep network, subsequently reducing its prediction over-confidence, thus improving  generalization performance.

% By bringing the class-conditional probability distributions closer together, we hope to make the decision boundary less sensitive to extreme points, and improve validation performance. Please see Figure~\ref{fig:algo_representation} for an abstract illustration of this idea. We describe the method in detail in the next sections.

% \begin{figure}[t]
% \includegraphics[width=\linewidth]{figures/algo_representation.pdf}
% \caption{When samples are scarce with high intra-class and low inter-class variance, we expect complete minimization of the training error to lead to decision boundaries that are sensitive to outliers (artifacts in images). With our algorithm, by bringing class-conditional output distributions closer to each other (depicted in the image as shaded areas color-coded according to class) we hypothesize the generalization ability of the classifiers will improve.}
% \label{fig:algo_representation}
% \end{figure}
Let us formalize the idea of {``confusing''} the conditional probability distributions. Consider the conditional probability distributions for two input images $\mathbf x_1$ and $\mathbf x_2$, which can be given by $p_\theta(\mathbf y | \mathbf x_1)$ and $p_\theta(\mathbf y | \mathbf x_2)$ respectively. For a classification problem with $N$ output classes, each of these distributions is an N-dimensional vector, with each element $i$ denoting the belief of the classifier in class $\mathbf y_i$ given input $\mathbf x$. If we wish to \textit{confuse} the class outputs of the classifier for the pair $\mathbf x_1$ and $\mathbf x_2$, we should learn parameters $\theta$ that bring these conditional probability distributions ``closer'' under some distance metric, that is, make the predictions for $\mathbf x_1$ and $\mathbf x_2$ similar.

While KL-divergence might seem to be a reasonable choice to design a loss function for optimizing the distance between conditional probability distributions, in Section~\ref{sec:3.1}, we show that it is infeasible to train a neural network when using KL-divergence as a regularizer. Therefore, we introduce the Euclidean Distance between distributions as a metric for confusion in Sections~\ref{sec:3.2} and~\ref{sec:3.3} and describe neural network training with this metric in Section~\ref{sec:3.4}.

\subsection{Symmetric KL-divergence or Jeffrey's Divergence\label{sec:3.1}}

The most prevalent method to measure dissimilarity of one probability distribution from another is to use the Kullback-Liebler (KL) divergence. However, the standard KL-divergence cannot serve our purpose owing to its asymmetric nature. This could be remedied by using the \textit{symmetric} KL-divergence, defined for two probability distributions $P, Q$ with mass functions $p(\cdot), q(\cdot)$ (for events $u \in \mathcal U$):
\begin{equation}
\mathbb D_{\mathsf J} (P, Q) \triangleq \sum_{u \in \mathcal U} \Big[p(u)\cdot\log\frac{p(u)}{q(u)}+q(u)\cdot\log\frac{q(u)}{p(u)}\Big] = \mathbb D_{\mathsf{KL}} (P || Q) + \mathbb D_{\mathsf{KL}} (Q || P)
\end{equation}
This symmetrized version of KL-divergence, known as Jeffrey's divergence~\cite{jeffreys1998theory}, is a measure of the average relative entropy between two probability distributions~\cite{kullback1951information}. For our model parameterized by $\theta$, for samples $\mathbf x_1$ and $\mathbf x_2$, the Jeffrey's divergence can be written as:
\begin{equation}
\mathbb D_{\mathsf J} (p_\theta(\mathbf y | \mathbf x_1), p_\theta(\mathbf y | \mathbf x_2)) = \sum_{i=1}^N \Big[(p_\theta(\mathbf y_i | \mathbf x_1) - p_\theta(\mathbf y_i | \mathbf x_2))\cdot\log\frac{p_\theta(\mathbf y_i | \mathbf x_1)}{p_\theta(\mathbf y_i | \mathbf x_2)}\Big] \label{eqn:jef2}
\end{equation}
Jeffrey's divergence satisfies all of our basic requirements of a symmetric divergence metric between probability distributions, and therefore could be included as a regularizing term while training with cross-entropy, to achieve our desired confusion. However, when we learn model parameters using stochastic gradient descent (SGD), it can be difficult to train, especially if our distributions $P, Q$ have mass concentrated on different events. This can be seen in Equation~\ref{eqn:jef2}. Consider Jeffrey's divergence with $N=2$ classes, and that $\mathbf x_1$ belongs to class 1, and $\mathbf x_2$ belongs to class 2. If the model parameters $\theta$ are such that it correctly identifies both $\mathbf x_1$ and $\mathbf x_2$ by training using cross-entropy loss, $p_\theta(\mathbf y_1 | x_1) = 1 - \delta_1$ and $p_\theta(\mathbf y_2 | x_2) = 1 - \delta_2$, where $0 < \delta_1, \delta_2 < \frac{1}{2}$ (since the classifier outputs correct predictions for the input images), we can show:
\begin{align}
\mathbb D_{\mathsf J} (p_\theta(\mathbf y | \mathbf x_1), p_\theta(\mathbf y | \mathbf x_2)) &\geq (1 - \delta_1 -\delta_2)\cdot(2\log(1-\delta_1-\delta_2) -\log(\delta_1\delta_2)) \label{eqn:jef3}
\end{align}
Please see the supplementary material for an expanded proof.

As training progresses with these labels, the cross-entropy loss will motivate the values of $\delta_1$ and $\delta_2$ to become closer to zero (but never equaling zero, since the probability outputs $p_\theta(\mathbf y | \mathbf x_1), p_\theta(\mathbf y | \mathbf x_2)$ are the outputs from a softmax). As $(\delta_1, \delta_2) \rightarrow (0^+, 0^+)$, the second term $-\log(\delta_1 \delta_2)$ on the R.H.S. of inequality~(\ref{eqn:jef3}) typically grows whereas $(1-\delta_1-\delta_2)$ approaches 1, which makes $\mathbb D_{\mathsf J} (p_\theta(\mathbf y | \mathbf x_1), p_\theta(\mathbf y | \mathbf x_2))$ larger as the predictions get closer to the true labels. In practice, we see that training with $\mathbb D_{\mathsf J} (p_\theta(\mathbf y | \mathbf x_1), p_\theta(\mathbf y | \mathbf x_2))$ as a regularizer term diverges, unless a very small regularizing parameter is chosen, which removes the effect of regularization altogether.

A natural question that can arise from this analysis is that cross-entropy training itself involves optimizing KL-divergence between the target label distribution and the model's predictions, however no such divergence occurs. This is because cross-entropy involves only one direction of the KL-divergence, and the target distribution has all the mass concentrated at one event (the correct label). Since $(x\log x) |_{x=0} = 0$, for predicted label vector $\mathbf y'$ with correct label class $c$, this simplifies the cross-entropy error $\mathcal L_{\mathsf{CE}}(p_\theta(\mathbf y | \mathbf x), \mathbf y')$ to be:
\begin{equation}
\mathcal L_{\mathsf{CE}}(p_\theta(\mathbf y | \mathbf x), \mathbf y') = - \sum_{i=1}^N \mathbf y'_i \log(\frac{p_\theta(\mathbf y_i | \mathbf x)}{\mathbf y'_i}) = - \log(p_\theta(\mathbf y_c | \mathbf x)) \geq 0
\end{equation}
This formulation does not diverge as the model trains, i.e. $p_\theta(\mathbf y_c | \mathbf x) \rightarrow 1$. In some cases where label noise is added to the label vector (such as label smoothing~\cite{reed2014training,szegedy2015going}), the label noise is a fixed constant and not approaching zero (as in the case of Jeffery's divergence between model predictions) and is hence feasible to train. Thus, Jeffrey's Divergence or symmetric KL-divergence, while a seemingly natural choice, cannot be used to train a neural network with SGD. This motivates us to look for an alternative metric to measure ``confusion'' between conditional probability distributions.

\subsection{Euclidean Distance as Confusion\label{sec:3.2}}
Since the conditional probability distribution over $N$ classes is an element within $\mathbb R^N$ on the unit simplex, we can consider the Euclidean distance to be a metric of ``confusion'' between two conditional probability distributions. Analogous to the previous setting, we define the \textbf{Euclidean Confusion} $\mathbb D_{\mathsf{EC}} (\cdot, \cdot)$ for a pair of inputs $\mathbf x_1, \mathbf x_2$ with model parameters $\theta$ as:
\begin{equation}
\mathbb D_{\mathsf{EC}}(p_\theta(\mathbf y | \mathbf x_1), p_\theta(\mathbf y | \mathbf x_2)) = \sum_{i=1}^N (p_\theta(\mathbf y_i | \mathbf x_1) - p_\theta(\mathbf y_i | \mathbf x_2))^2 = \lVert p_\theta(\mathbf y | \mathbf x_1) - p_\theta(\mathbf y | \mathbf x_2)  \rVert_2^2
\end{equation}
Unlike Jeffrey's Divergence, Euclidean Confusion does not diverge when used as a regularization term with cross-entropy. However, to verify this unconventional choice for a distance metric between probability distributions, we prove some properties that relate Euclidean Confusion to existing divergence measures.
\begin{lemma}
\label{lemma_1}
On a finite probability space, the Euclidean Confusion $\mathbb D_{\mathsf{EC}}(P, Q)$ is a lower bound for the Jeffrey's Divergence $\mathbb D_{\mathsf J}(P, Q)$ for probability measures $P, Q$.
\end{lemma}
\begin{proof}
This follows from Pinsker's Inequality and the relationship between $\ell_1$ and $\ell_2$ norms. Complete proof is provided in the supplementary material.
\end{proof}
By Lemma~\ref{lemma_1}, we can see that the Euclidean Confusion is a conservative estimate for Jeffrey's divergence, the earlier proposed divergence measure. For finite probability spaces, the Total Variation Distance $\mathbb D_{\mathsf {TV}}(P, Q)^2 = \frac{1}{2} \lVert P - Q \rVert_1$ is also a measure of interest. However, due to its non-differentiable nature, it is unsuitable for our case. Nevertheless, we can relate the Euclidean Confusion and Total Variation Distance by the following result.
\begin{lemma}
On a finite probability space, the Euclidean Confusion $\mathbb D_{\mathsf{EC}}(P, Q)$ is bounded by $4\mathbb D_{\mathsf{TV}}(P, Q)^2$ for probability measures $P, Q$.
\end{lemma}
\begin{proof}
This follows directly from the relationship between $\ell_1$ and $\ell_2$ norms. Complete proof is provided in the supplementary material.
\end{proof}

\subsection{Euclidean Confusion for Point Sets\label{sec:3.3}}
In a standard classification setting with $N$ classes, we consider a training set with $m = \sum_{i=1}^N m_i$ training examples, where $m_i$ denotes the number of training samples for class $i$. For this setting, we can write the total Euclidean Confusion between points of classes $i$ and $j$ as the average of the Euclidean Confusion between all pairs of points belonging to those two classes. For simplicity of notation, let us denote the set of conditional probability distributions of all training points belonging to class $i$ for a model parameterized by $\theta$ as $\mathcal S_i = \{p_\theta(\mathbf y|\mathbf x^i_1), p_\theta(\mathbf y|\mathbf x^i_2), ..., p_\theta(\mathbf y|\mathbf x^i_{m_i})\}$. Then, for a model parameterized by $\theta$, the Euclidean Confusion is given by:
{\footnotesize
\begin{align}
\mathbb D_{\mathsf{EC}}(\mathcal S_i, \mathcal S_j; \theta) \triangleq \frac{1}{m_im_j} \Big(\sum_{u,v}^{m_i,m_j} \mathbb D_{\mathsf{EC}}(p_\theta(\mathbf y |\mathbf x^i_u), p_\theta(\mathbf y|\mathbf x^j_v)) \Big)
% = \frac{1}{m_im_j}\Big(\sum_{u,v}^{m_i,m_j} \lVert p_\theta(\mathbf y |\mathbf x^i_u) - p_\theta(\mathbf y|\mathbf x^j_v) \rVert_2^2 \Big)
\end{align}
}%
We can simplify this equation by assuming an equal number of points $n$ per class:
\begin{equation}
\mathbb D_{\mathsf{EC}}(\mathcal S_i, \mathcal S_j; \theta) = \frac{1}{n^2}\Big(\sum_{u,v}^{n,n} \lVert p_\theta(\mathbf y |\mathbf x^i_u) - p_\theta(\mathbf y|\mathbf x^j_v) \rVert_2^2 \Big) \label{eqn:set_ec}
\end{equation}
This form of the Euclidean Confusion between the two sets of points gives us an interesting connection with another popular distance metric over probability distributions, known as the \textbf{Energy Distance}~\cite{szekely2013energy}.

Introduced by Gabor Szekely~\cite{szekely2013energy}, the \textbf{Energy Distance} $\mathbb D_{\mathsf{EN}}(F,G)$ between two cumulative probability distribution functions $F$ and $G$ with random vectors $X$ and $Y$ in $\mathbb R^N$ can be given by
\begin{equation}
\mathbb D_{\mathsf{EN}}(F,G)^2 \triangleq 2\mathbb E \lVert X - Y\rVert - \mathbb E\lVert X - X'\rVert -\mathbb E\lVert Y - Y'\rVert \geq 0
\end{equation}
where $(X,X',Y,Y')$ are independent, and $X \sim F, X' \sim F, Y \sim G, Y' \sim G$. If we consider the sets $\mathcal S_i$ and $\mathcal S_j$, with a uniform probability of selecting any of the $n$ points in each of these sets, then we obtain the following results.
\begin{lemma}
  For sets $\mathcal S_i$, $\mathcal S_j$ and $\mathbb D_{\mathsf{EC}}(\mathcal S_i, \mathcal S_j; \theta)$ as defined in Equation~(\ref{eqn:set_ec}):
  \begin{equation*}
    \tfrac{1}{2}\mathbb D_{\mathsf{EN}}(\mathcal S_i, \mathcal S_j; \theta)^2 \leq \mathbb D_{\mathsf{EC}}(\mathcal S_i, \mathcal S_j; \theta)
  \end{equation*}
  where $\mathbb D_{\mathsf{EN}}(\mathcal S_i, \mathcal S_j; \theta)$ is the Energy Distance under Euclidean norm between $\mathcal S_i$ and $\mathcal S_j$ (parameterized by $\theta$), and random vectors are selected with uniform probability in both $\mathcal S_i$ and $\mathcal S_j$.
\end{lemma}
\begin{proof}
  This follows from the definition of Energy Distance with uniform probability of sampling. Complete proof is provided in the supplementary material.
\end{proof}
\begin{corollary}
  For sets $\mathcal S_i$, $\mathcal S_j$ and $\mathbb D_{\mathsf{EC}}(\mathcal S_i, \mathcal S_j; \theta)$ as defined in Equation~(\ref{eqn:set_ec}), we have:
  \begin{equation*}
    \mathbb D_{\mathsf{EC}}(\mathcal S_i, \mathcal S_i; \theta) + \mathbb D_{\mathsf{EC}}(\mathcal S_j, \mathcal S_j; \theta) \leq 2\mathbb D_{\mathsf{EC}}(\mathcal S_i, \mathcal S_j; \theta)
  \end{equation*}
  with equality only when $\mathcal S_i = \mathcal S_j$.
\end{corollary}
\begin{proof}
  This follows from the fact that the Energy Distance $\mathbb D_{\mathsf{EN}}(\mathcal S_i, \mathcal S_j; \theta)$ is 0 only when $\mathcal S_i = \mathcal S_j$. The complete version of the proof is included in the supplement.
\end{proof}
With these results, we restrict the behavior of Euclidean Confusion within two well-defined conventional probability distance measures, the Jeffrey's divergence and Energy Distance. One might consider optimizing the Energy Distance directly, due to its similar formulation and the fact that we uniformly sample points during training with SGD. However, the Energy Distance additionally includes the two terms that account for the negative of the average all-pairs distances between points in $\mathcal S_i$ and $\mathcal S_j$ respectively, which we do not want to maximize, since we do not wish to push points within the same class further apart. Therefore, we proceed with our measure of Euclidean Confusion.
\begin{figure*}[t]
\includegraphics[width=\linewidth]{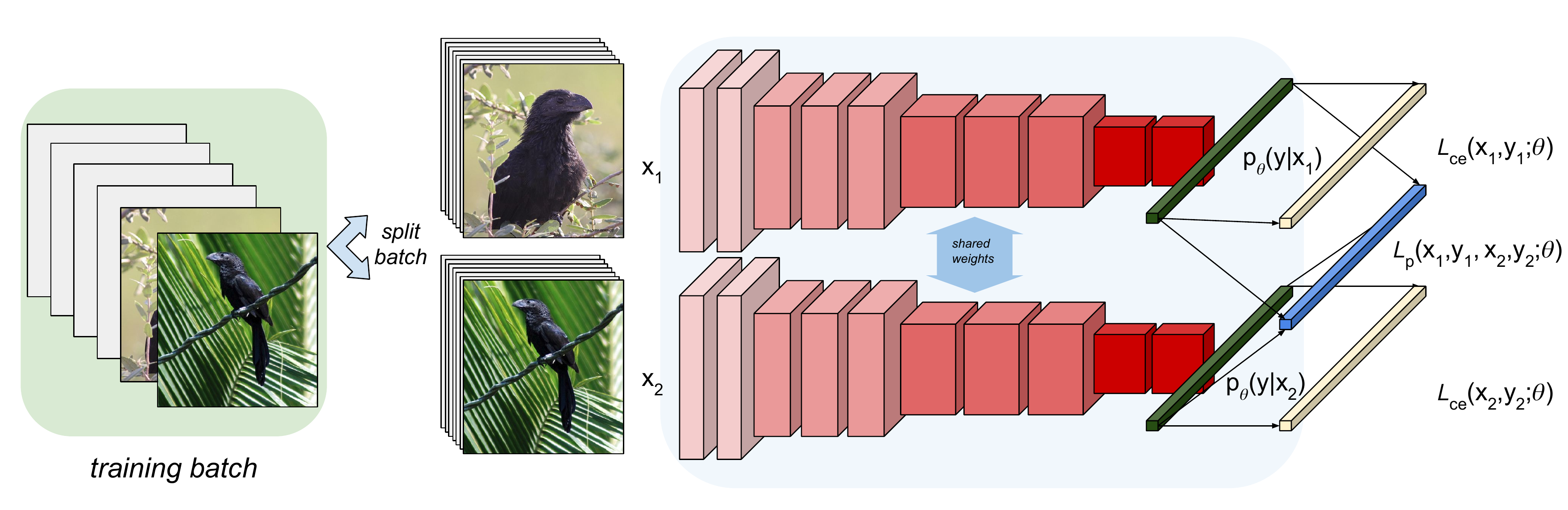}
\caption{CNN training pipeline for Pairwise Confusion (\textbf{PC}). We employ a Siamese-like architecture, with individual cross entropy calculations for each branch, followed by a joint energy-distance minimization loss. We split each incoming batch of samples into two mini-batches, and feed the network pairwise samples.}
\label{fig:model_architecture}
\end{figure*}

\subsection{Learning with Gradient Descent\label{sec:3.4}}
We proceed to learn parameters $\theta^*$ for a neural network, with the following learning objective function for a pair of input points, motivated by the formulation of Euclidean Confusion:
{
\small
\begin{equation}
  \theta^* = \arg\min_{\theta} \sum_{\substack{i=1, j\neq i\\ u, v}}^{\substack{N,N \\ n, n}} \Big[\mathcal L_{\textsf{CE}}(p_\theta(\mathbf y | \mathbf x^i_u), \mathbf y^i_u) + \mathcal L_{\textsf{CE}}(p_\theta(\mathbf y | \mathbf x^j_v), \mathbf y^j_v) + \frac{\lambda}{n^2} \mathbb D_{\mathsf{EC}}(p_\theta(\mathbf y | \mathbf x^j_v), p_\theta(\mathbf y | \mathbf x^i_u)) \Big] \label{optim_equation}
\end{equation}
}%
This objective function can be explained as: for each point in the training set, we randomly select another point from a different class and calculate the individual cross-entropy losses and Euclidean Confusion until all pairs have been exhausted. For each point in the training dataset, there are $n{\cdot}(N-1)$ valid choices for the other point, giving us a total of $n^2{\cdot}N{\cdot}(N-1)$ possible pairs. In practice, we find that we do not need to exhaust all combinations for effective learning using gradient descent, and in fact we observe that convergence is achieved far before all observations are observed. We simplify our formulation instead by using the following procedure described in Algorithm~\ref{alg1}.

{\footnotesize
\begin{algorithm} % enter the algorithm environment
\caption{Training Using Euclidean Confusion} % give the algorithm a caption
\label{alg1} % and a label for \ref{} commands later in the document
\begin{algorithmic} % enter the algorithmic environment
	\State Training data $D$, Test data $\hat{D}$, parameters $\theta$, hyperparameters $\hat{\theta}$
	\For{$epoch \in [0,$\textsf{max\textunderscore epochs}]$)$}
		\State $D_1 \Leftarrow \textsf{shuffle}(D)$
		\State $D_2 \Leftarrow \textsf{shuffle}(D)$
    \For{$i \in [0,$\textsf{num\textunderscore batches}$]$}
      \State{$\mathcal L_{\textsf{batch}} = 0$}
      \For{$(d_1, d_2) \in $ \textsf{batch} $i$ of $(D_1, D_2)$}
        \State{$\gamma \Leftarrow 1$ if \textsf{label}$(d_1)\neq$ \textsf{label}$(d_2)$, 0 otherwise}
        \State{$\mathcal L_{\textsf{pair}} \Leftarrow \mathcal L_{\textsf{CE}}(d_1; \theta) + \mathcal L_{\textsf{CE}}(d_2; \theta) + \lambda \cdot \gamma \cdot \mathbb D_{\textsf{EC}}(d_1, d_2; \theta)$}
        \State{$\mathcal L_{\textsf{batch}} \Leftarrow \mathcal L_{\textsf{batch}} + \mathcal L_{\textsf{pair}}$}
      \EndFor
      \State{$\theta \Leftarrow$ \textsf{Backprop}$(\mathcal L_{\textsf{batch}}, \theta, \hat{\theta})$}
    \EndFor
    \State{$\hat\theta \Leftarrow$ \textsf{ParameterUpdate}(\textsf{epoch}, $\hat\theta)$}
	\EndFor
\end{algorithmic}
\end{algorithm}
}%

%\vspace{-15pt}
\subsubsection{Training Procedure:}
As described in Algorithm 1, our learning procedure is a slightly modified version of the standard SGD. We randomly permute the training set twice, and then for each pair of points in the training set, add Euclidean Confusion only if the samples belong to different classes. This form of sampling approximates the exhaustive Euclidean Confusion, with some points with regular gradient descent, which in practice does not alter the performance. Moreover, convergence is achieved after only a fraction of all the possible pairs are observed. Formally, we wish to model the conditional probability distribution $p_\theta(\mathbf y | \mathbf x)$ over the $p$ classes for function $f(\mathbf x ; \theta) = p_\theta(\mathbf y | \mathbf x)$ parameterized by model parameters $\theta$. Given our optimization procedure, we can rewrite the total loss for a pair of points $\mathbf x_1, \mathbf x_2$ with model parameters $\theta$ as:
{
%\vspace{-5pt}
\small
\begin{equation}
%\begin{gather*}
  \mathcal L_{\textsf{pair}}(\mathbf x_1, \mathbf x_2, \mathbf y_1, \mathbf y_2; \theta) = \sum_{i=1}^2 [\mathcal L_{\textsf{CE}}(p_\theta(\mathbf y | \mathbf x_i), \mathbf y_i)] + \lambda\gamma(\mathbf y_1, \mathbf y_2)\mathbb D_{\mathsf{EC}}(p_\theta(\mathbf y | \mathbf x_1), p_\theta(\mathbf y | \mathbf x_2))
%\end{gather*}
\end{equation}} where, $\gamma(\mathbf y_1, \mathbf y_2) = 1$ when $\mathbf y_i \neq \mathbf y_j$, and 0 otherwise. We denote \textbf{training} with this general architecture with the term \textit{Pairwise Confusion} or \textbf{PC} for short. Specifically, we train a Siamese-like neural network~\cite{chopra2005learning} with shared weights, training each network individually using cross-entropy, and add the \textbf{Euclidean Confusion} loss between the conditional probability distributions obtained from each network (Figure \ref{fig:model_architecture}). During training, we split an incoming batch of training samples into two parts, and evaluating cross-entropy on each sub-batch identically, followed by a pairwise loss term calculated for corresponding pairs of samples across batches. During {testing}, only one branch of the network is active, and generates output predictions for the input image. As a result, implementing this method does not introduce any significant computational overhead during testing.

% \subsubsection{Siamese Networks}
% Siamese networks were introduced by Chopra et al.~\cite{chopra2005learning} to model the learning of discriminative metrics on pairwise samples. They have been applied to a variety of pairwise learning problems in computer vision \cite{souri2016deep,singh2016end,dubey2017modeling}. A standard Siamese network consists of two identical neural networks which share weights, operating on two different images. The outputs of these neural networks are then combined using a pairwise loss function. We describe the training and evaluation protocol in the next section.
\subsubsection{CNN Architectures}
We experiment with VGGNet \cite{simonyan2014very}, GoogLeNet \cite{szegedy2015going}, ResNets \cite{he2016deep}, and DenseNets~\cite{huang2016densely} as base architectures for the Siamese network trained with \textbf{PC} to demonstrate that our method is insensitive to the choice of source architecture.
%Following prior work~\cite{gao2016compact,lin2015bilinear,cui2017kernel}, we train using transfer learning; the CNNs are initialized with ImageNet-trained weights and fine-tuned on the target datasets. We additionally perform experiments on CNNs tailored for FGVC tasks, including Bilinear CNNs~\cite{lin2015bilinear}.

%% file: S_experiments.tex
\section{Experimental Details}
We perform all experiments using Caffe~\cite{jia2014caffe} or PyTorch \cite{pytorch} over a cluster of NVIDIA Titan X, Tesla K40c and  GTX 1080 GPUs. Our code and models are available at \href{https://github.com/abhimanyudubey/confusion}{\texttt{github.com/abhimanyudubey/confusion}}. Next, we provide brief descriptions of the various datasets used in our paper.
\subsection{Fine-Grained Visual Classification (FGVC) datasets}
\begin{enumerate}[wide, labelwidth=!, labelindent=0pt]
	\item \textbf{Wildlife Species Classification}: We experiment with several widely-used FGVC datasets. The Caltech-UCSD Birds (\textbf{CUB-200-2011)} dataset~\cite{wah2011caltech} has 5,994 training and 5,794 test images across 200 species of North-American birds. The \textbf{NABirds} dataset~\cite{van2015building} contains 23,929 training and 24,633 test images across over 550 visual categories, encompassing 400 species of birds, including separate classes for male and female birds in some cases. The \textbf{Stanford Dogs} dataset~\cite{khosla2011novel} has 20,580 images across 120 breeds of dogs around the world. Finally, the \textbf{Flowers-102} dataset~\cite{nilsback2008automated} consists of 1,020 training, 1,020 validation and 6,149 test images over 102 flower types.
	\item \textbf{Vehicle Make/Model Classification}:  We experiment with two common vehicle classification datasets. The \textbf{Stanford Cars} dataset~\cite{krause20133d} contains 8,144 training and 8,041 test images across 196 car classes. The classes represent variations in car {make}, {model}, and {year}. The \textbf{Aircraft} dataset is a set of 10,000 images across 100 classes denoting a fine-grained set of airplanes of different varieties~\cite{maji2013fine}.
\end{enumerate}

{\begin{table*}[t]
\centering
\scriptsize
\setlength\tabcolsep{0.75pt}
\caption{Pairwise Confusion (\textbf{PC}) obtains state-of-the-art performance on six widely-used fine-grained visual classification datasets (A-F). Improvement over the baseline model is reported as $(\Delta)$. All results averaged over 5 trials.\label{tab:fgvc}}
\begin{tabular}{lcc}
\multicolumn{3}{c}{(A) CUB-200-2011} \\ \hline
Method & Top-1 & $\Delta$ \\ \hline
%Simon \textit{et al.} \cite{simon2015neural} & {81.10}\\
%Krause \textit{et al.} \cite{krause2015fine} & {82.00}\\
Gao \textit{et al.} \cite{gao2016compact}&84.00 & -\\
STN\cite{jaderberg2015spatial} & {84.10} & - \\
Zhang \textit{et al.} \cite{zhang2016picking} & 84.50 & - \\
Lin \textit{et al.}~\cite{lin2017improved} & 85.80 & - \\
Cui \textit{et al.}~\cite{cui2017kernel} & 86.20 & - \\\hline
% GoogLeNet & 68.19 & \multirow{2}{*}{(\textbf{4.46})} \\
% \textbf{PC}-GoogLeNet & 72.65 & \\ \hline
ResNet-50 & 78.15 & \multirow{2}{*}{(2.06)} \\
\textbf{PC}-ResNet-50 & 80.21 &  \\\hline
%VGGNet16 &73.28 & \multirow{2}{*}{(3.20)} \\
%\textbf{PC}-VGGNet16 & 76.48 & \\ \hline
Bilinear CNN~\cite{lin2015bilinear} &84.10 & \multirow{2}{*}{(1.48)} \\
\textbf{PC}-BilinearCNN & 85.58 & \\ \hline
DenseNet-161 & 84.21 & \multirow{2}{*}{(\textbf{2.66})} \\
\textbf{PC}-DenseNet-161 & \textbf{86.87} \\ \hline\\
\end{tabular}
\hspace{2pt}
\begin{tabular}{lcc}
\multicolumn{3}{c}{(B) Cars} \\ \hline
Method & Top-1 & $\Delta$ \\ \hline
Wang \textit{et al.} \cite{Wang_2016_CVPR} & 85.70 & - \\
Liu \textit{et al.} \cite{liu2016hierarchical} & 86.80 & - \\
Lin \textit{et al.}~\cite{lin2017improved} & 92.00 & - \\
Cui \textit{et al.}~\cite{cui2017kernel} & 92.40 & - \\ \hline
% GoogLeNet & 85.65 & \multirow{2}{*}{(1.26)}\\
% \textbf{PC}-GoogLeNet & 86.91 & \\\hline
ResNet-50 & 91.71 & \multirow{2}{*}{(1.72)}\\
\textbf{PC}-ResNet-50 & \textbf{93.43}  \\ \hline
%VGGNet16 &80.60 & \multirow{2}{*}{(\textbf{2.56})} \\
%\textbf{PC}-VGGNet16 & 83.16 \\ \hline
Bilinear CNN~\cite{lin2015bilinear} &91.20 & \multirow{2}{*}{(\textbf{1.25})}\\
\textbf{PC}-Bilinear CNN & 92.45  \\ \hline
DenseNet-161 & 91.83 & \multirow{2}{*}{(1.03)} \\
\textbf{PC}-DenseNet-161 & 92.86 &  \\ \hline\\
\end{tabular}
\hspace{2pt}
\begin{tabular}{lcc}
\multicolumn{3}{c} {(C) Aircrafts} \\ \hline
Method & Top-1 & $\Delta$ \\ \hline
Simon \textit{et al.} \cite{simon2017generalized} & 85.50 & -\\
Cui \textit{et al.}~\cite{cui2017kernel} & 86.90 & - \\
LRBP \cite{kong2016low} & 87.30 & -\\
Lin \textit{et al.}~\cite{lin2017improved} & 88.50 & - \\ \hline
% GoogLeNet & 74.04 & \multirow{2}{*}{(\textbf{4.82})} \\
% \textbf{PC}-GoogLeNet & 78.86 \\\hline
ResNet-50 & 81.19 & \multirow{2}{*}{(2.21)} \\
\textbf{PC}-ResNet-50 & 83.40  \\ \hline
%VGGNet16 &74.17 & \multirow{2}{*}{(3.03)}\\
%\textbf{PC}-VGGNet16 & 77.20 \\ \hline
BilinearCNN~\cite{lin2015bilinear} &84.10 & \multirow{2}{*}{(1.68)}\\
\textbf{PC}-BilinearCNN & 85.78 \\ \hline
DenseNet-161 &86.30 & \multirow{2}{*}{(\textbf{2.94})} \\
\textbf{PC}-DenseNet-161 & \textbf{89.24} \\ \hline\\
\end{tabular}
%%\vspace{30pt}
\begin{tabular}{lcc}
\multicolumn{3}{c}{(D) NABirds} \\ \hline
Method & Top-1 & $\Delta$ \\ \hline
Branson \textit{et al.}~\cite{branson2014bird} & 35.70 & - \\
Van \textit{et al.} \cite{van2015building} & 75.00 & -\\ \hline %\tablefootnote{Obtained with part annotations.}
% GoogLeNet & 70.66 & \multirow{2}{*}{(1.35)}\\
% \textbf{PC}-GoogLeNet & 72.01 & \\ \hline
ResNet-50~& 63.55 & \multirow{2}{*}{(\textbf{4.60})} \\
\textbf{PC}-ResNet-50 & 68.15 & \\  \hline
%VGGNet16 &68.34 & \multirow{2}{*}{(3.91)}\\
%\textbf{PC}-VGGNet16 & 72.25 &  \\ \hline
BilinearCNN~\cite{lin2015bilinear} & 80.90 & \multirow{2}{*}{(1.11)} \\
\textbf{PC}-BilinearCNN & 82.01 & \\ \hline
DenseNet-161 & 79.35 & \multirow{2}{*}{(3.44)}\\
\textbf{PC}-DenseNet-161 & \textbf{82.79} \\ \hline
\end{tabular}
\hspace{1pt}
\begin{tabular}{lcc}
\multicolumn{3}{c} {(E) Flowers-102} \\ \hline
Method & Top-1 & $\Delta$ \\ \hline
Det.+Seg. \cite{angelova2013efficient} & 80.66 & -\\
Overfeat\cite{Razavian_2014_CVPR_Workshops} & 86.80 & -\\ \hline
% GoogLeNet & 82.55 & \multirow{2}{*}{(0.48)}\\
% \textbf{PC}-GoogLeNet & 83.03 & \\\hline
ResNet-50 & 92.46 & \multirow{2}{*}{(1.04)} \\
\textbf{PC}-ResNet-50 & 93.50  \\ \hline
%VGGNet16 &85.15 & \multirow{2}{*}{(1.04)}\\
%\textbf{PC}-VGGNet16 & 86.19 \\ \hline
BilinearCNN~\cite{lin2015bilinear} &92.52 & \multirow{2}{*}{(1.13)}\\
\textbf{PC}-BilinearCNN & \textbf{93.65} \\\hline
DenseNet-161 & 90.07 & \multirow{2}{*}{(\textbf{1.32})}\\
\textbf{PC}-DenseNet-161 &91.39 & \\ \hline
\end{tabular}
\hspace{1pt}
\begin{tabular}{lcc}
\multicolumn{3}{c}{(F) Stanford Dogs} \\ \hline
Method & Top-1 & $\Delta$ \\ \hline
%Tu \textit{et al.} \cite{tu2010one} & 79.50 \\
Zhang \textit{et al.} \cite{zhang2016weakly} & 80.43 & - \\
Krause \textit{et al.} \cite{krause2016unreasonable} & 80.60 & - \\ \hline
% GoogLeNet & 55.76 & \multirow{2}{*}{(\textbf{4.85})} \\
% \textbf{PC}-GoogLeNet & 60.61 & \\\hline
ResNet-50 & 69.92 &  \multirow{2}{*}{(\textbf{3.43})}\\
\textbf{PC}-ResNet-50 & 73.35  \\\hline
%VGGNet16 &61.92 &  \multirow{2}{*}{(3.59)}\\
%\textbf{PC}-VGGNet16 & 65.51 & \\ \hline
BilinearCNN~\cite{lin2015bilinear} &82.13 & \multirow{2}{*}{(0.91)}\\
\textbf{PC}-BilinearCNN & 83.04 \\  \hline
DenseNet-161 & 81.18 &  \multirow{2}{*}{(2.57)}\\
\textbf{PC}-DenseNet-161 & \textbf{83.75} & \\ \hline
\end{tabular}
\end{table*}

These datasets contain (i) large visual diversity in each class~\cite{nilsback2008automated,wah2011caltech,khosla2011novel}, (ii) visually similar, often confusing samples belonging to different classes, and (iii) a large variation in the number of samples present per class, leading to greater class imbalance than  LSVC datasets like \textbf{ImageNet}~\cite{imagenet_cvpr09}. Additionally, some of these datasets have densely annotated part information available, which we do not utilize in our experiments.

%% file: S_results.tex
\section{Results}
\label{sec:results}
\subsection{Fine-Grained Visual Classification}
We first describe our results on the six FGVC datasets from Table~\ref{tab:fgvc}. In all experiments, we average results over 5 trials per experiment---after choosing the best value of hyperparameter $\lambda$. Please see the supplementary material for  mean and standard deviation values for all experiments.
\begin{enumerate}[wide, labelwidth=!, labelindent=0pt]
\item \textbf{Fine-tuning from Baseline Models}: We fine-tune from three baseline models using the PC optimization procedure: ResNet-50~\cite{he2016deep}, Bilinear CNN~\cite{lin2015bilinear}, and DenseNet-161~\cite{huang2016densely}. As Tables~\ref{tab:fgvc}-(A-F) show, PC obtains substantial improvement across all datasets and models. For instance, a baseline DenseNet-161 architecture obtains an average accuracy of 84.21\%, but \textbf{PC}-DenseNet-161 obtains an accuracy of 86.87\%, an improvement of \textbf{2.66\%}. On NABirds, we obtain improvements of \textbf{4.60\%} and \textbf{3.42\%} over baseline ResNet-50 and DenseNet-161 architectures.

\item \textbf{Combining PC with Specialized FGVC models}: Recent work in FGVC has proposed several novel CNN designs that take part-localization into account, such as bilinear pooling techniques~\cite{gao2016compact,lin2015bilinear,cui2017kernel} and spatial transformer networks~\cite{jaderberg2015spatial}. We train a Bilinear CNN~\cite{lin2015bilinear} with PC, and obtain an average improvement of 1.7\% on the 6 datasets.
\end{enumerate}
We note two important aspects of our analysis: (1) we do not compare with ensembling and data augmentation techniques such as Boosted CNNs~\cite{Moghimi2016} and Krause, \textit{et al.}~\cite{krause2016unreasonable} since prior evidence indicates that these techniques invariably improve performance, and (2) we evaluate a single-crop, single-model evaluation without any part- or object-annotations, and perform competitively with methods that use both augmentations.
\begin{figure}[t]
\centering
\begin{minipage}[t]{.45\textwidth}
\centering
\includegraphics[width=\linewidth]{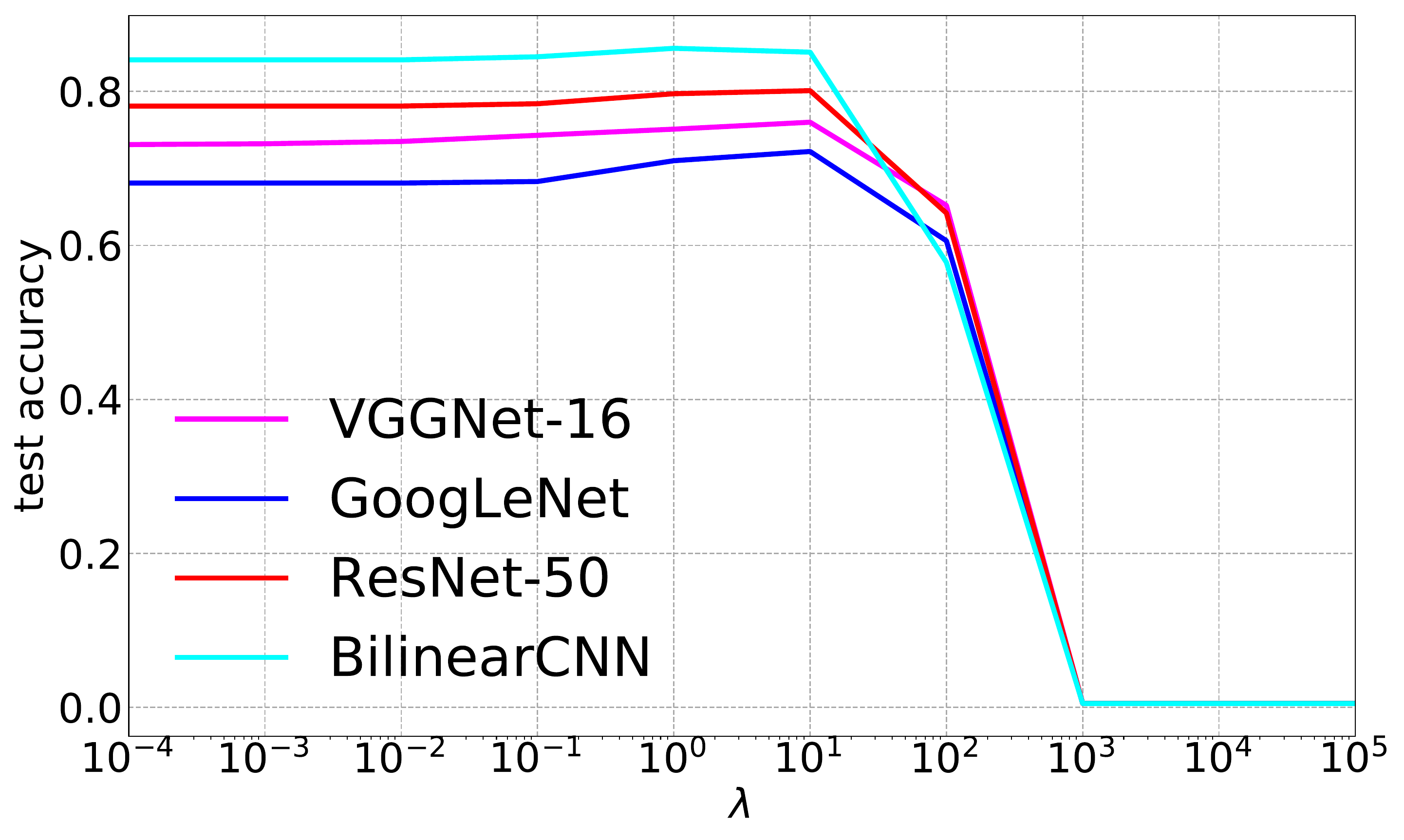}
\end{minipage}
\begin{minipage}[t]{.45\textwidth}
\centering
\includegraphics[width=\linewidth]{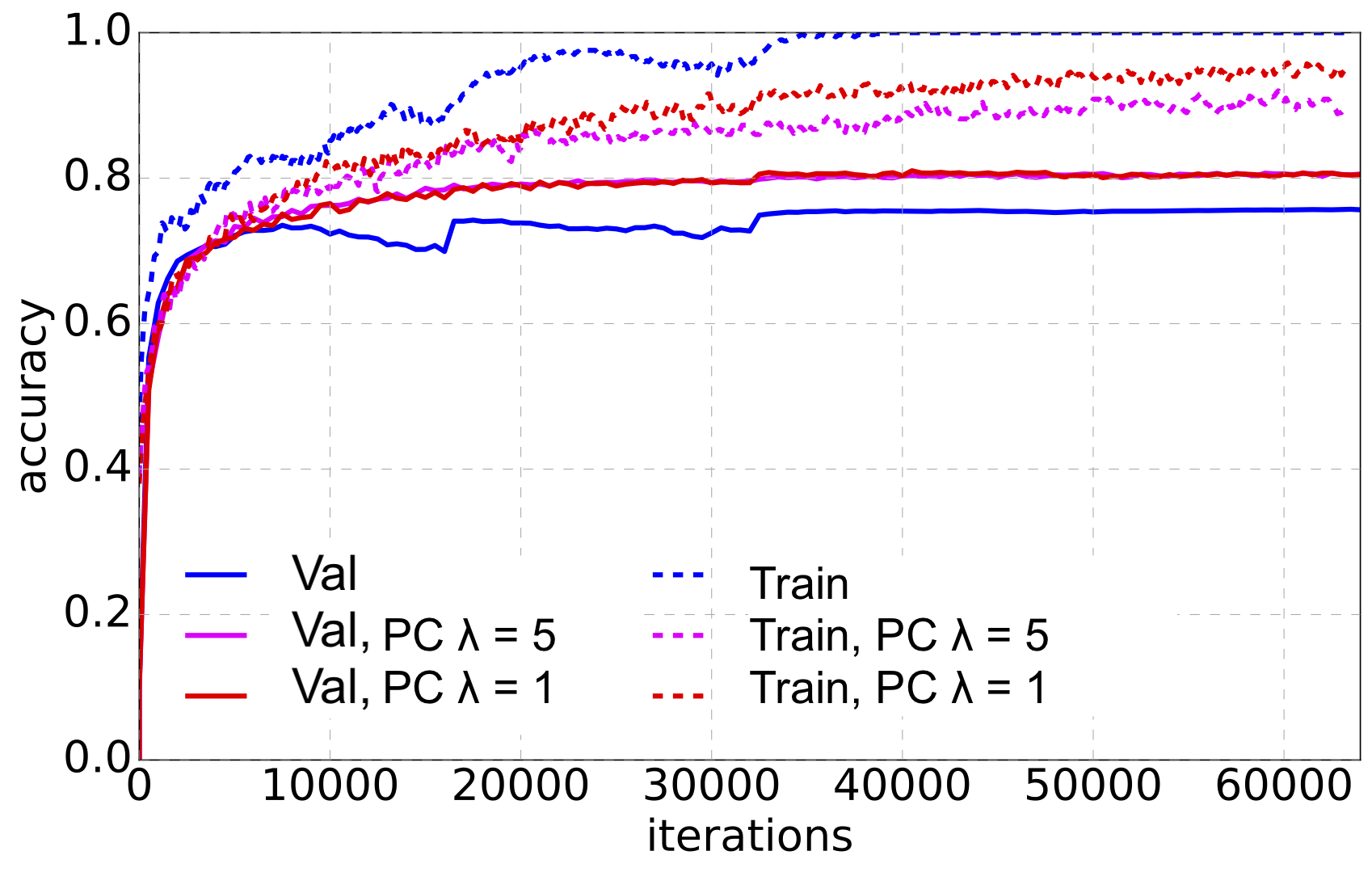}
\end{minipage}
\caption{(left) Variation of test accuracy on CUB-200-2011 with logarithmic variation in hyperparameter $\lambda$. (right) Convergence plot of GoogLeNet on CUB-200-2011.}
\label{fig:lambda_cub}
\end{figure}

\textbf{Choice of Hyperparameter $\lambda$:} Since our formulation requires the selection of a hyperparameter $\lambda$, it is important to study the sensitivity of classification performance to the choice of $\lambda$. We conduct this experiment for four different models: GoogLeNet~\cite{szegedy2015going}, ResNet-50~\cite{he2016deep} and VGGNet-16~\cite{simonyan2014very} and Bilinear-CNN~\cite{lin2015bilinear} on the CUB-200-2011 dataset. PC's performance is \textbf{not very sensitive to the choice of $\lambda$} (Figure~\ref{fig:lambda_cub} and Supplementary Tables S1-S5). For all six datasets, the $\lambda$ value is typically between the range [10,20]. On Bilinear CNN, setting $\lambda = 10$ for all datasets gives average performance within 0.08\% compared to the reported values in Table~\ref{tab:fgvc}. In general, PC obtains optimum performance in the range of $0.05N$ and $0.15N$, where N is the number of classes.

\subsection{Additional Experiments}
Since our method aims to improve classification performance in FGVC tasks by introducing confusion in output logit activations, we would expect to see a larger improvement in datasets with higher inter-class similarity and intra-class variation. To test this hypothesis, we conduct two additional experiments.

In the first experiment, we construct two subsets of {ImageNet-1K}~\cite{imagenet_cvpr09}. The first dataset, \textbf{ImageNet-Dogs} is a subset consisting only of species of dogs (117 classes and 116K images). The second dataset, \textbf{ImageNet-Random} contains randomly selected classes from {ImageNet-1K}. Both datasets contain equal number of classes (117) and images (116K), but ImageNet-Dogs has much higher inter-class similarity and intra-class variation, as compared to ImageNet-Random. To test repeatability, we construct 3 instances of Imagenet-Random, by randomly choosing a different subset of ImageNet with 117 classes each time. For both experiments, we randomly construct a 80-20 train-val split from the training data to find optimal $\lambda$ by cross-validation, and report the performance on the unseen ImageNet validation set of the subset of chosen classes. In Table~\ref{tab:ablation}, we compare the performance of training from scratch with- and without-PC across three models: GoogLeNet, ResNet-50, and DenseNet-161. As expected, PC obtains a larger gain in classification accuracy (1.45\%) on ImageNet-Dogs as compared to the {ImageNet-Random} dataset($0.54\% \pm 0.28$).

In the second experiment, we utilize the CIFAR-10 and CIFAR-100 datasets, which contain the same number of total images. CIFAR-100 has $10\times$ the number of classes and $10\%$ of images per class as CIFAR-10 and contains larger inter-class similarity and intra-class variation.  We train networks on both datasets from scratch using default train-test splits (Table~\ref{tab:ablation}). As expected, we obtain larger average gains of 1.77\% on CIFAR-100, as compared to 0.20\% on CIFAR-10. Additionally, when training with $\lambda=10$ on the entire ImageNet dataset, we obtain a top-1 accuracy of $76.28\%$ (compared to a baseline of $76.15\%$), which is a smaller improvement, which is in line with what we would expect for a large-scale image classification problem with large inter-class variation.

Moreover, while training with PC, we observe that the rate of convergence is always similar to or faster than training without PC. For example, a GoogLeNet trained on CUB-200-2011 (Figure~\ref{fig:lambda_cub}(right) above) shows that PC converges to higher validation accuracy faster than normal training using identical learning rate schedule and batch size. Note that the training accuracy is reduced when training with PC, due to the regularization effect. In sum, classification problems that have large intra-class variation and high inter-class similarity benefit from optimization with pairwise confusion. The improvement is even more prominent when training data is limited.

% \begin{table}
% \centering\small
% \begin{tabular}{l|c|c}
% \hline \hline
% \multirow{2}{*}{\textbf{Method}} & \multicolumn{2}{c}{\textbf{Mean IoU}} \\ \cline{2-3}
% & Baseline & \textbf{PC} \\ \hline
% GoogLeNet~\cite{szegedy2015going} & 0.29 & 0.35 \\
% VGGNet-16~\cite{simonyan2014very} & 0.31 & 0.34 \\
% ResNet-50~\cite{he2016deep} & 0.32 & 0.35 \\
% Bilinear-CNN~\cite{lin2015bilinear} & 0.37 & 0.39 \\
% DenseNet-161~\cite{huang2016densely} & 0.34 & 0.37 \\ \hline
% \end{tabular}
% \caption{PC improves localization performance in fine-grained visual classification tasks. On the CUB-200-2011 dataset, PC obtains an average improvement of 3.4\% in Mean Intersection-over-Union (IoU) for Grad-CAM bounding boxes over five baseline models.}
% \label{tab:iou_gradcam}
% \end{table}

\begin{table}[t]
\centering
\small
\caption{Experiments with ImageNet and CIFAR show that datasets with large intra-class variation and high inter-class similarity benefit from optimization with Pairwise Confusion. Only the mean accuracy over 3 Imagenet-Random experiments is shown.\label{tab:ablation}}
\scalebox{1}{
\begin{tabular}{l|c|c|c|c|c|c|c|c} \hline \hline
\multirow{2}{*}{Network}  & \multicolumn{2}{c|}{{ImageNet-Random}} & \multicolumn{2}{c|}{{ImageNet-Dogs}} &  \multicolumn{2}{c|}{{CIFAR-10}} & \multicolumn{2}{c}{{CIFAR-100}} \\ \cline{2-9}
& Baseline & {PC} & Baseline & {PC} & Baseline & {PC} & Baseline & {PC} \\ \hline
GoogLeNet~\cite{szegedy2015going}  & 71.85 & 72.09 & 62.35 & 64.17 &  86.63 & 87.02 & 73.35 & 76.02\\
ResNet-50~\cite{he2016deep} & 82.01 & 82.65 & 73.81 & 75.92 & 93.17 & 93.46 & 72.16 & 73.14 \\
DenseNet-161~\cite{huang2016densely}  & 78.34 & 79.10 & 70.15 & 71.44 & 95.15 & 95.08 & 78.60 & 79.56\\
\hline \hline
\end{tabular}}
%\vspace{-10pt}
\end{table}

\begin{table}[t]
\centering
\small
\caption{Pairwise Confusion (\textbf{PC}) improves localization performance in fine-grained visual classification tasks. On the CUB-200-2011 dataset, PC obtains an average improvement of 3.4\% in Mean Intersection-over-Union (IoU) for Grad-CAM bounding boxes for each of the five baseline models.\label{tab:iou_gradcam}}
\begin{tabular}{l|c|c|c|c|c}
\hline \hline
Method               & GoogLeNet & VGG-16 & ResNet-50 & DenseNet-161 & Bilinear-CNN \\\hline
Mean IoU (Baseline)  & 0.29      & 0.31   & 0.32      & 0.34         & 0.37         \\
Mean IoU (PC) - Ours & 0.35      & 0.34   & 0.35      & 0.37         & 0.39      \\\hline\hline
\end{tabular}
\end{table}

\begin{figure}[t]
\centering
\small
\includegraphics[width=\textwidth]{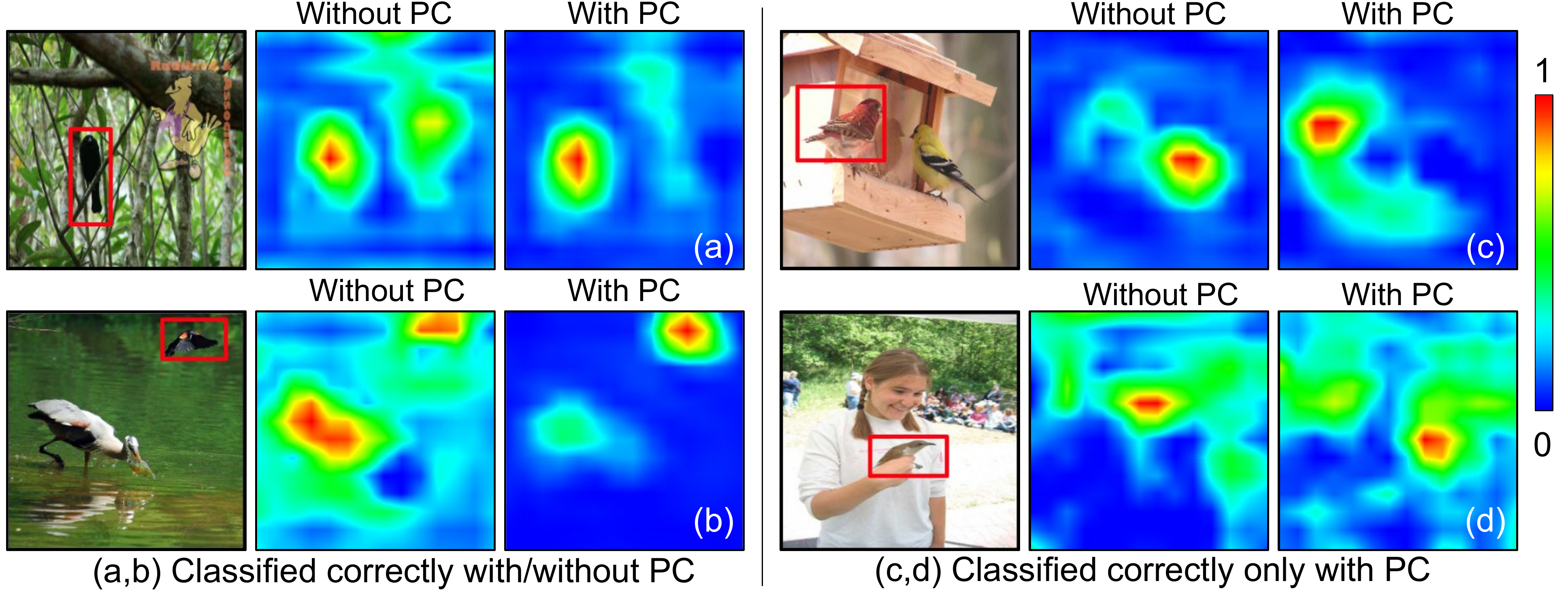}
\caption{Pairwise Confusion (\textbf{PC}) obtains improved localization performance, as demonstrated here with Grad-CAM heatmaps of the CUB-200-2011 dataset images (left) with a VGGNet-16 model trained without PC (middle) and with PC (right). The objects in (a) and (b) are correctly classified by both networks, and (c) and (d) are correctly classified by {PC}, but not the baseline network (VGG-16). For all cases, we consistently observe a tighter and more accurate localization with PC, whereas the baseline VGG-16 network  often latches on to artifacts, even while making correct predictions.}
\label{fig:grad_cam}
\end{figure}
\subsection{Improvement in Localization Ability}
Recent techniques for improving classification performance in fine-grained recognition are based on summarizing and extracting dense localization information in images~\cite{lin2015bilinear,jaderberg2015spatial}. Since our technique increases classification accuracy, we wish to understand if the improvement is a result of enhanced CNN localization abilities due to PC. To measure the regions the CNN localizes on, we utilize Gradient-Weighted Class Activation Mapping (Grad-CAM)~\cite{selvaraju2016grad}, a method that provides a heatmap of visual saliency as produced by the network. We perform both quantitative and qualitative studies of localization ability of PC-trained models.\\
\textbf{Overlap in Localized Regions:} To quantify the improvement in localization due to PC, we construct bounding boxes around object regions obtained from Grad-CAM, by thresholding the heatmap values at 0.5, and choosing the largest box returned. We then calculate the mean IoU (intersection-over-union) of the bounding box with the provided object bounding boxes for the CUB-200-2011 dataset. We compare the mean IoU across several models, with and without PC. As summarized in Table~\ref{tab:iou_gradcam}, we observe an average 3.4\% improvement across five different networks, implying better localization accuracy.\\
\textbf{Change in Class-Activation Mapping:} To qualitatively study the improvement in localization due to PC, we obtain samples from the CUB-200-2011 dataset and visualize the localization regions returned from Grad-CAM for both the baseline and PC-trained VGG-16 model. As shown in Figure~\ref{fig:grad_cam}, PC models provide tighter, more accurate  localization around the target object, whereas sometimes the baseline model has localization driven by image artifacts. Figure~\ref{fig:grad_cam}-(a) has an example of the types of distractions that are often present in FGVC images (the cartoon bird on the right). We see that the baseline VGG-16 network pays significant attention to the distraction, despite making the correct prediction. With {PC}, we find that the attention is limited almost exclusively to the correct object, as desired. Similarly for Figure~\ref{fig:grad_cam}-(b), we see that the baseline method latches on to the incorrect bird category, which is corrected by the addition of {PC}. In Figures~\ref{fig:grad_cam}-(c-d), we see that the baseline classifier makes incorrect decisions due to poor localization, mistakes that are resolved by {PC}.

\section{Conclusion}
In this work, we introduce Pairwise Confusion (PC), an optimization procedure to improve generalizability in fine-grained visual classification (FGVC) tasks by encouraging confusion in output activations. PC improves FGVC performance for a wide class of convolutional architectures while fine-tuning. Our experiments indicate that PC-trained networks show improved localization performance which contributes to the gains in classification accuracy. PC is easy to implement, does not need excessive tuning during training, and does not add significant overhead during test time, in contrast to methods that introduce complex localization-based pooling steps that are often difficult to implement and train. Therefore, our technique should be beneficial to a wide variety of specialized neural network models for applications that demand for fine-grained visual classification or learning from limited labeled data.

\textbf{Acknowledgements:} We would like to thank Dr. Ashok Gupta for his guidance on bird recognition, and Dr. Sumeet Agarwal, Spandan Madan and Ishaan Grover for their feedback at various stages of this work.

%% file: main.bbl
\begin{thebibliography}{10}

\bibitem{lin2015bilinear}
Lin, T.Y., RoyChowdhury, A., Maji, S.:
\newblock Bilinear cnn models for fine-grained visual recognition.
\newblock IEEE International Conference on Computer Vision (2015)  1449--1457

\bibitem{jaderberg2015spatial}
Jaderberg, M., Simonyan, K., Zisserman, A., Kavukcuoglu, K.:
\newblock Spatial transformer networks.
\newblock Advances in Neural Information Processing Systems (2015)  2017--2025

\bibitem{zhang2016weakly}
Zhang, Y., Wei, X.S., Wu, J., Cai, J., Lu, J., Nguyen, V.A., Do, M.N.:
\newblock Weakly supervised fine-grained categorization with part-based image
  representation.
\newblock IEEE Transactions on Image Processing \textbf{25}(4) (2016)
  1713--1725

\bibitem{krause2015fine}
Krause, J., Jin, H., Yang, J., Fei-Fei, L.:
\newblock Fine-grained recognition without part annotations.
\newblock IEEE Conference on Computer Vision and Pattern Recognition (2015)
  5546--5555

\bibitem{zhang2015fine}
Zhang, N., Shelhamer, E., Gao, Y., Darrell, T.:
\newblock Fine-grained pose prediction, normalization, and recognition.
\newblock International Conference on Learning Representations Workshops (2015)

\bibitem{krause2016unreasonable}
Krause, J., Sapp, B., Howard, A., Zhou, H., Toshev, A., Duerig, T., Philbin,
  J., Fei-Fei, L.:
\newblock The unreasonable effectiveness of noisy data for fine-grained
  recognition.
\newblock European Conference on Computer Vision (2016)  301--320

\bibitem{cui2016fine}
Cui, Y., Zhou, F., Lin, Y., Belongie, S.:
\newblock Fine-grained categorization and dataset bootstrapping using deep
  metric learning with humans in the loop.
\newblock IEEE Conference on Computer Vision and Pattern Recognition (2016)

\bibitem{lin2017improved}
Lin, T.Y., Maji, S.:
\newblock Improved bilinear pooling with cnns.
\newblock arXiv preprint arXiv:1707.06772 (2017)

\bibitem{cui2017kernel}
Cui, Y., Zhou, F., Wang, J., Liu, X., Lin, Y., Belongie, S.:
\newblock Kernel pooling for convolutional neural networks.
\newblock IEEE Conference on Computer Vision and Pattern Recognition (2017)

\bibitem{imagenet_cvpr09}
Deng, J., Dong, W., Socher, R., Li, L.J., Li, K., Fei-Fei, L.:
\newblock Imagenet: A large-scale hierarchical image database.
\newblock IEEE Conference on Computer Vision and Pattern Recognition (2009)
  248--255

\bibitem{huang2016densely}
Huang, G., Liu, Z., van~der Maaten, L., Weinberger, K.Q.:
\newblock Densely connected convolutional networks.
\newblock IEEE Conference on Computer Vision and Pattern Recognition (2017)

\bibitem{he2016deep}
He, K., Zhang, X., Ren, S., Sun, J.:
\newblock Deep residual learning for image recognition.
\newblock IEEE Conference on Computer Vision and Pattern Recognition (2016)
  770--778

\bibitem{yao2011combining}
Yao, B., Khosla, A., Fei-Fei, L.:
\newblock Combining randomization and discrimination for fine-grained image
  categorization.
\newblock IEEE Conference on Computer Vision and Pattern Recognition (2011)
  1577--1584

\bibitem{yao2012codebook}
Yao, B., Bradski, G., Fei-Fei, L.:
\newblock A codebook-free and annotation-free approach for fine-grained image
  categorization.
\newblock IEEE Conference on Computer Vision and Pattern Recognition (2012)
  3466--3473

\bibitem{zhang2014part}
Zhang, N., Donahue, J., Girshick, R., Darrell, T.:
\newblock Part-based r-cnns for fine-grained category detection.
\newblock European Conference on Computer Vision (2014)  834--849

\bibitem{gao2016compact}
Gao, Y., Beijbom, O., Zhang, N., Darrell, T.:
\newblock Compact bilinear pooling.
\newblock IEEE Conference on Computer Vision and Pattern Recognition (2016)
  317--326

\bibitem{Wang_2016_CVPR}
Wang, Y., Choi, J., Morariu, V., Davis, L.S.:
\newblock Mining discriminative triplets of patches for fine-grained
  classification.
\newblock IEEE Conference on Computer Vision and Pattern Recognition (June
  2016)

\bibitem{ren2015faster}
Ren, S., He, K., Girshick, R., Sun, J.:
\newblock Faster r-cnn: Towards real-time object detection with region proposal
  networks.
\newblock Advances in neural information processing systems (2015)  91--99

\bibitem{branson2014bird}
Branson, S., Van~Horn, G., Belongie, S., Perona, P.:
\newblock Bird species categorization using pose normalized deep convolutional
  nets.
\newblock British Machine Vision Conference (2014)

\bibitem{zhang2012pose}
Zhang, N., Farrell, R., Darrell, T.:
\newblock Pose pooling kernels for sub-category recognition.
\newblock IEEE Computer Vision and Pattern Recognition (2012)  3665--3672

\bibitem{Moghimi2016}
Moghimi, M., Saberian, M., Yang, J., Li, L.J., Vasconcelos, N., Belongie, S.:
\newblock Boosted convolutional neural networks.
\newblock (2016)

\bibitem{chopra2005learning}
Chopra, S., Hadsell, R., LeCun, Y.:
\newblock Learning a similarity metric discriminatively, with application to
  face verification.
\newblock IEEE Conference on Computer Vision and Pattern Recognition (2005)
  539--546

\bibitem{parikh2011relative}
Parikh, D., Grauman, K.:
\newblock Relative attributes.
\newblock IEEE International Conference on Computer Vision (2011)  503--510

\bibitem{dubey2017modeling}
Dubey, A., Agarwal, S.:
\newblock Modeling image virality with pairwise spatial transformer networks.
\newblock arXiv preprint arXiv:1709.07914 (2017)

\bibitem{souri2016deep}
Souri, Y., Noury, E., Adeli, E.:
\newblock Deep relative attributes.
\newblock Asian Conference on Computer Vision (2016)  118--133

\bibitem{dubey2016deep}
Dubey, A., Naik, N., Parikh, D., Raskar, R., Hidalgo, C.A.:
\newblock Deep learning the city: Quantifying urban perception at a global
  scale.
\newblock European Conference on Computer Vision (2016)  196--212

\bibitem{singh2016end}
Singh, K.K., Lee, Y.J.:
\newblock End-to-end localization and ranking for relative attributes.
\newblock European Conference on Computer Vision (2016)  753--769

\bibitem{reed2014training}
Reed, S., Lee, H., Anguelov, D., Szegedy, C., Erhan, D., Rabinovich, A.:
\newblock Training deep neural networks on noisy labels with bootstrapping.
\newblock arXiv preprint arXiv:1412.6596 (2014)

\bibitem{xiao2015learning}
Xiao, T., Xia, T., Yang, Y., Huang, C., Wang, X.:
\newblock Learning from massive noisy labeled data for image classification.
\newblock IEEE Conference on Computer Vision and Pattern Recognition (2015)
  2691--2699

\bibitem{neelakantan2015adding}
Neelakantan, A., Vilnis, L., Le, Q.V., Sutskever, I., Kaiser, L., Kurach, K.,
  Martens, J.:
\newblock Adding gradient noise improves learning for very deep networks.
\newblock arXiv preprint arXiv:1511.06807 (2015)

\bibitem{szegedy2016rethinking}
Szegedy, C., Vanhoucke, V., Ioffe, S., Shlens, J., Wojna, Z.:
\newblock Rethinking the inception architecture for computer vision.
\newblock IEEE Conference on Computer Vision and Pattern Recognition (2016)
  2818--2826

\bibitem{nilsback2008automated}
Nilsback, M.E., Zisserman, A.:
\newblock Automated flower classification over a large number of classes.
\newblock Indian Conference on Computer Vision, Graphics \& Image Processing
  (2008)  722--729

\bibitem{wah2011caltech}
Wah, C., Branson, S., Welinder, P., Perona, P., Belongie, S.:
\newblock The caltech-ucsd birds-200-2011 dataset.
\newblock (2011)

\bibitem{krause20133d}
Krause, J., Stark, M., Deng, J., Fei-Fei, L.:
\newblock 3d object representations for fine-grained categorization.
\newblock IEEE International Conference on Computer Vision Workshops (2013)
  554--561

\bibitem{van2015building}
Van~Horn, G., Branson, S., Farrell, R., Haber, S., Barry, J., Ipeirotis, P.,
  Perona, P., Belongie, S.:
\newblock Building a bird recognition app and large scale dataset with citizen
  scientists: The fine print in fine-grained dataset collection.
\newblock IEEE Conference on Computer Vision and Pattern Recognition (2015)
  595--604

\bibitem{maji2013fine}
Maji, S., Rahtu, E., Kannala, J., Blaschko, M., Vedaldi, A.:
\newblock Fine-grained visual classification of aircraft.
\newblock arXiv preprint arXiv:1306.5151 (2013)

\bibitem{khosla2011novel}
Khosla, A., Jayadevaprakash, N., Yao, B., Li, F.F.:
\newblock Novel dataset for fine-grained image categorization: Stanford dogs.
\newblock IEEE International Conference on Computer Vision Workshops on
  Fine-Grained Visual Categorization (2011) ~1

\bibitem{krizhevsky2014cifar}
Krizhevsky, A., Nair, V., Hinton, G.:
\newblock The cifar-10 dataset otkrist (2014)

\bibitem{netzer2011reading}
Netzer, Y., Wang, T., Coates, A., Bissacco, A., Wu, B., Ng, A.Y.:
\newblock Reading digits in natural images with unsupervised feature learning.
\newblock NIPS workshop on deep learning and unsupervised feature learning (2)
  (2011) ~5

\bibitem{jeffreys1998theory}
Jeffreys, H.:
\newblock The theory of probability.
\newblock OUP Oxford (1998)

\bibitem{kullback1951information}
Kullback, S., Leibler, R.A.:
\newblock On information and sufficiency.
\newblock The annals of mathematical statistics \textbf{22}(1) (1951)  79--86

\bibitem{szegedy2015going}
Szegedy, C., Liu, W., Jia, Y., Sermanet, P., Reed, S., Anguelov, D., Erhan, D.,
  Vanhoucke, V., Rabinovich, A.:
\newblock Going deeper with convolutions.
\newblock IEEE Conference on Computer Vision and Pattern Recognition (2015)
  1--9

\bibitem{szekely2013energy}
Sz{\'e}kely, G.J., Rizzo, M.L.:
\newblock Energy statistics: A class of statistics based on distances.
\newblock Journal of statistical planning and inference \textbf{143}(8) (2013)
  1249--1272

\bibitem{simonyan2014very}
Simonyan, K., Zisserman, A.:
\newblock Very deep convolutional networks for large-scale image recognition.
\newblock arXiv preprint arXiv:1409.1556 (2014)

\bibitem{jia2014caffe}
Jia, Y., Shelhamer, E., Donahue, J., Karayev, S., Long, J., Girshick, R.,
  Guadarrama, S., Darrell, T.:
\newblock Caffe: Convolutional architecture for fast feature embedding.
\newblock ACM international conference on Multimedia (2014)  675--678

\bibitem{pytorch}
Paskze, A., Chintala, S.:
\newblock {Tensors and Dynamic neural networks in Python with strong GPU
  acceleration}.
\newblock \url{https://github.com/pytorch} Accessed: [January 1, 2017].

\bibitem{zhang2016picking}
Zhang, X., Xiong, H., Zhou, W., Lin, W., Tian, Q.:
\newblock Picking deep filter responses for fine-grained image recognition.
\newblock IEEE Conference on Computer Vision and Pattern Recognition (2016)
  1134--1142

\bibitem{liu2016hierarchical}
Liu, M., Yu, C., Ling, H., Lei, J.:
\newblock Hierarchical joint cnn-based models for fine-grained cars
  recognition.
\newblock International Conference on Cloud Computing and Security (2016)
  337--347

\bibitem{simon2017generalized}
Simon, M., Gao, Y., Darrell, T., Denzler, J., Rodner, E.:
\newblock Generalized orderless pooling performs implicit salient matching.
\newblock International Conference on Computer Vision (ICCV) (2017)

\bibitem{kong2016low}
Kong, S., Fowlkes, C.:
\newblock Low-rank bilinear pooling for fine-grained classification.
\newblock IEEE Conference on Computer Vision and Pattern Recognition (2017)
  7025--7034

\bibitem{angelova2013efficient}
Angelova, A., Zhu, S.:
\newblock Efficient object detection and segmentation for fine-grained
  recognition.
\newblock IEEE Conference on Computer Vision and Pattern Recognition (2013)
  811--818

\bibitem{Razavian_2014_CVPR_Workshops}
Sharif~Razavian, A., Azizpour, H., Sullivan, J., Carlsson, S.:
\newblock C{NN} features off-the-shelf: An astounding baseline for recognition.
\newblock IEEE Conference on Computer Vision and Pattern Recognition Workshops
  (June 2014)

\bibitem{selvaraju2016grad}
Selvaraju, R.R., Das, A., Vedantam, R., Cogswell, M., Parikh, D., Batra, D.:
\newblock Grad-cam: Why did you say that? visual explanations from deep
  networks via gradient-based localization.
\newblock arXiv preprint arXiv:1610.02391 (2016)

\bibitem{krogh1991simple}
Krogh, A., Hertz, J.A.:
\newblock A simple weight decay can improve generalization.
\newblock NIPS \textbf{4} (1991)  950--957

\bibitem{cogswell2015reducing}
Cogswell, M., Ahmed, F., Girshick, R., Zitnick, L., Batra, D.:
\newblock Reducing overfitting in deep networks by decorrelating
  representations.
\newblock arXiv preprint arXiv:1511.06068 (2015)

\bibitem{srivastava2014dropout}
Srivastava, N., Hinton, G.E., Krizhevsky, A., Sutskever, I., Salakhutdinov, R.:
\newblock Dropout: a simple way to prevent neural networks from overfitting.
\newblock Journal of Machine Learning Research \textbf{15}(1) (2014)
  1929--1958

\end{thebibliography}
